\documentclass[11pt]{article}

\newif\ifcolt\colttrue
\coltfalse

\ifcolt

\title[A Second-Order Method for Stochastic Bandit Convex Optimisation]{A Second-Order Method for Stochastic Bandit Convex Optimisation}

\coltauthor{%
 \Name{Tor Lattimore} \Email{lattimore@deepmind.com} \\
 \addr DeepMind, London
 \AND
 \Name{Andr\'as Gy\"orgy} \Email{agyorgy@deepmind.com} \\
 \addr DeepMind, London
}

\else
\title{\large \textsc{A Second-Order Method for Stochastic Bandit Convex Optimisation}}
\author{\small Tor Lattimore and Andr\'as Gy\"orgy}
\date{\small DeepMind, London}

\fi

\ifcolt
\else
\usepackage{fullpage}
\fi

\ifcolt
\crefname{equation}{Eq.}{Eqs.}
\fi

\usepackage{xcolor}
\definecolor{dkblue}{cmyk}{1,.54,.04,.19} 

\usepackage[textsize=tiny,disable]{todonotes}

\usepackage{hyperref} % add backlinks to the end of bibliography items

\ifcolt
\hypersetup{
  bookmarks=true,         % show bookmarks bar?
    unicode=false,          % non-Latin characters in AcrobatÕs bookmarks
    pdftoolbar=true,        % show AcrobatÕs toolbar?
    pdfmenubar=true,        % show AcrobatÕs menu?
    pdffitwindow=false,     % window fit to page when opened
    pdfstartview={FitH},    % fits the width of the page to the window
    pdftitle={A Second-Order Method for Stochastic Bandit Convex Optimisation},    % title
    pdfauthor={},     % author
    pdfsubject={Bandits},   % subject of the document
    pdfcreator={pdflatex},   % creator of the document
    pdfproducer={Producer}, % producer of the document
    pdfkeywords={bandits} {zeroth-order convex optimisation} {machine learning}, % list of keywords
    pdfnewwindow=true,      % links in new window
    colorlinks=true,        % false: boxed links; true: colored links
    linkcolor=black,       % color of internal links
    citecolor=dkblue,       % color of links to bibliography
    filecolor=dkblue,       % color of file links
    urlcolor=dkblue,        % color of external links
}
\else
\hypersetup{
  bookmarks=true,         % show bookmarks bar?
    unicode=false,          % non-Latin characters in AcrobatÕs bookmarks
    pdftoolbar=true,        % show AcrobatÕs toolbar?
    pdfmenubar=true,        % show AcrobatÕs menu?
    pdffitwindow=false,     % window fit to page when opened
    pdfstartview={FitH},    % fits the width of the page to the window
    pdftitle={A Second-Order Method for Stochastic Bandit Convex Optimisation},    % title
    pdfauthor={Tor Lattimore and Andr\'as Gy\"orgy},
    pdfsubject={Bandits},   % subject of the document
    pdfcreator={pdflatex},   % creator of the document
    pdfproducer={Producer}, % producer of the document
    pdfkeywords={bandits} {zeroth-order convex optimisation} {machine learning}, % list of keywords
    pdfnewwindow=true,      % links in new window
    colorlinks=true,        % false: boxed links; true: colored links
    linkcolor=black,       % color of internal links
    citecolor=dkblue,       % color of links to bibliography
    filecolor=dkblue,       % color of file links
    urlcolor=dkblue,        % color of external links
}

\fi

\ifcolt
\usepackage{times}
\else
\usepackage{times}
\fi
\usepackage{pgfplots}
\pgfplotsset{compat=1.11}
\usepackage{pgfplotstable}
\usepackage{nicefrac}
\usepackage{colortbl}
\usepackage{floatrow}
\usepackage{amsmath}
\usepackage{setspace}
\usepackage{mathtools}
\usepackage{listings}
\usepackage{mathrsfs}
\usepackage{tikz}
\usepackage{dsfont}
\usepackage{amssymb}
\usepackage[boxed]{algorithm}
\usepackage{bm}
\usepackage{cleveref}
\usepackage[bf]{caption}
\usepackage{graphicx}
\usepackage{enumitem}

\setlist[enumerate,1]{label={\textit{(\alph*)}},itemsep=0pt,wide,labelwidth=!,labelindent=0pt}

\newlist{enumeratenum}{enumerate}{1}
\setlist[enumeratenum]{label={\textit{(\arabic*)}},itemsep=0pt}

\let\epsilon\varepsilon

\ifcolt

\else

\usepackage{amsthm}
\theoremstyle{plain}
\newtheorem{theorem}{Theorem}

\newtheorem{lemma}[theorem]{Lemma}

\newtheorem{corollary}[theorem]{Corollary}

\theoremstyle{definition}
\newtheorem{definition}[theorem]{Definition}

\theoremstyle{remark}
\fi

\usepackage{natbib}

\newcommand{\ceil}[1]{\left\lceil #1 \right\rceil}

\newcommand{\R}{\mathbb R}

\newcommand{\ip}[1]{\langle #1 \rangle}

\newcommand{\Reg}{\textrm{Reg}}

\newcommand{\Fmax}{\textrm{F}_{\max}}
\newcommand{\sphere}{\mathbb{S}^{d-1}}
\newcommand{\norm}[1]{\Vert #1 \Vert}
\newcommand{\normt}[2]{\Vert #2 \Vert_{#1}}
\newcommand{\bnorm}[1]{\left\Vert #1 \right\Vert}
\newcommand{\bignorm}[1]{\bigl\Vert #1 \bigr\Vert}
\newcommand{\E}{\mathbb E}

\newcommand{\cE}{\mathcal E}
\newcommand{\Cnst}{\textrm{C}}
\newcommand{\cnst}{\textrm{c}}

\newcommand{\cF}{\mathscr F}
\newcommand{\cC}{\mathcal C}

\newcommand{\hesst}{H_t}
\newcommand{\gradt}{g_t}
\newcommand{\cN}{\mathcal N}

\newcommand{\tr}{\operatorname{tr}}

\newcommand{\zeros}{ \bm 0}

\newcommand{\bbP}{\mathbb P}
\newcommand{\polylog}{\operatorname{polylog}}
\newcommand{\const}{\operatorname{const}}

\newcommand{\diam}{r}
\newcommand{\poly}{\operatorname{poly}}

\newcommand{\sind}{\bm{1}}
\renewcommand{\d}[1]{\operatorname{d}\!#1}
\newcommand{\id}{\mathds{1}}

\newcommand{\Dmax}{\textrm{D}_{\max}}
\newcommand{\Wmax}{\textrm{W}_{\max}}
\newcommand{\cnstP}{\textrm{P}}

\newcommand{\eref}[1]{\textit{#1}}
\newcommand{\er}[1]{#1}
\newcommand{\A}{A}
\newcommand{\B}{B}

\newtheorem{fact}[theorem]{Fact}

\begin{document}

\maketitle

\begin{abstract}
We introduce a simple and efficient algorithm for unconstrained zeroth-order stochastic convex bandits
and prove its regret is at most $(1 + r/d) [d^{1.5} \sqrt{n} + d^3] \polylog(n, d, r)$ where $n$ is the horizon, $d$ the dimension
and $r$ is the radius of a known ball containing the minimiser of the loss.
\end{abstract}

\ifcolt
\begin{keywords}%
Bandits; zeroth-order convex optimisation.
\end{keywords}
\fi

\section{Introduction}
Let $\norm{\cdot}$ be the standard euclidean norm and $f : \R^d \to \R$ be convex and assume that
\begin{enumerate}
\item $f$ is Lipschitz: $f(x) - f(y) \leq \norm{x - y}$ for all $x, y \in \R^d$;
\item there exists an $x_\star \in \R^d$ such that $f(x_\star) = \inf_{x \in \R^d} f(x)$;
\item the learner has access to a constant $\diam \geq 1$ and initial point $x_{\circ}$ such that $\norm{x_\star - x_{\circ}} \leq \diam$.
\end{enumerate}
A learner interacts with an environment over $n$ rounds.  
In each round $t$ the learner chooses $X_t \in \R^d$
and observes $Y_t = f(X_t) + \epsilon_t$ where $(\epsilon_t)_{t=1}^n$ is a sequence of conditionally zero mean subgaussian random variables (precise
condition given in \Cref{eq:subgauss} below).
As usual in bandit problems, $X_t$ is only allowed to depend on previous observations $X_1,Y_1,\ldots,X_{t-1},Y_{t-1}$ and possibly
an external source of randomness.
Our focus is on the cumulative regret: $\Reg_n = \sum_{t=1}^n f(X_t) - f(x_\star)$.
The main contribution is the following regret guarantee for a simple algorithm for which the computation
per round is dominated by finding the eigendecomposition of a $d\times d$ matrix.

\begin{theorem}\label{thm:main}
With probability at least $1 - 7 / n$, 
the regret of \Cref{alg:cvx} is upper bound by
\begin{align*}
\Reg_n \leq \const \left(1 + \frac{\diam}{d}\right)\left[d^{1.5} \sqrt{n} + d^3\right] \left(1 + \log \max\left(n, d, \diam\right)\right)^{4}
\,,
\end{align*}
where $\const$ is a universal constant.
\end{theorem}

The best known bound in this setting is $\E[\Reg_n] \leq \diam d^{2.5} \sqrt{n} \polylog(n, d, \diam)$, 
which does not come with an efficient algorithm \citep{Lat20-cvx}.
Our high-level idea is to apply continuous exponential weights
on the space of Gaussian probability measures in combination with 
a quadratic approximation of a surrogate loss function that is roughly the same as used by \cite{BEL16} and \cite{LG21a}.
Working directly on the space of Gaussian distributions with quadratic losses avoids the many complications involving exponential weights
distributions and approximate log-concave sampling needed for the algorithm of \cite{BEL16}.
Although our analysis and algorithm are designed for the regret setting, an important consequence is an improved bound 
for stochastic zeroth-order convex optimisation. 

\begin{corollary}\label{cor:main}
For $\epsilon > 0$ and $n \geq \ceil{\const \left(\frac{d^3}{\epsilon^2} + \frac{d^3}{\epsilon}\right) \Big(1 + \frac{\diam}{d}\Big)^2  \left(1 + \log (\max(1/\epsilon, d, \diam))\right)^8}$,
\begin{align*}
\bbP\left(f\left(\frac{1}{n} \sum_{t=1}^n X_t\right) \geq f(x_\star) + \epsilon\right) \leq \frac{7}{n}\,,
\end{align*}
where $X_1,\ldots,X_n$ are the actions chosen by \Cref{alg:cvx}.
\end{corollary}

Until now, the best known bound on the sample complexity of an efficient algorithm in this setting was $\tilde O(\frac{d^{7.5}}{\epsilon^2})$
by \cite{BLN15}. \cite{Lat20-cvx} demonstrated the existence of a procedure for which the sample complexity is at most 
$\smash{\tilde O(\frac{d^5}{\epsilon^2})}$, but the approach is non-constructive. We emphasise that both of these works are intended for the harder
constrained setting.

\paragraph{Related work}
There is an ever-growing literature on convex bandits in a variety of settings. 
Our setup is unusual because there are no constraints on the domain of the function to be optimised.
Of course, algorithms that handle constraints can be used in our setup because of the assumption that the minimum lies in a known ball.
The other direction is not clear. We expect that suitable modifications of our ideas will lead 
to algorithms for the constrained case, but not without
effort, ingenuity and possibly some dimension-dependent cost. More on this in the discussion.

The most natural idea to extend the standard machinery for stochastic gradient descent to the zeroth-order bandit setting is to use importance-weighted gradient estimators of a smoothed approximation of $f$, which
was the approach taken by \cite{Kle04}, \cite{FK05} and \cite{Sah11}.
This leads to simple generalisations of gradient descent that are straightforward to implement and analyse. 
Sadly this approach does not lead to $\sqrt{n}$ regret without strong convexity.

In the stochastic setting it is possible to adapt tools from the classical zeroth-order optimisation literature as was
shown by \cite{AFHK13}, who proved $\poly(d) \sqrt{n}$ regret for constrained stochastic convex bandits without smoothness or strong convexity assumptions. These ideas were improved by \cite{LG21a} leading to a better dimension dependence.
\cite{BDKP15} showed that $\sqrt{n}$ regret is also possible without strong convexity/smoothness in the adversarial 
setting when $d = 1$. Their approach non-constructively leveraged the information-theoretic machinery of \cite{RV14} and
did not yield an algorithm.
There followed a flurry of results generalising this to higher dimensions and/or polynomial time algorithms 
\citep{HL16,BEL16,BE18,Lat20-cvx}.
None of these algorithms is particularly straightforward to implement. 

What is missing in the literature is a simple algorithm with $\sqrt{n}$ regret in any setting without strong convexity.
Interestingly, \citet{HPGySz16:BCO} proved a negative result showing that any analysis that uses gradient estimators must use more properties of these estimators than 
any naive bias-variance decomposition that appeared in previous work \citep{Kle04,FK05}.
This negative result does not hold in the strongly convex setting, where gradient-based 
methods give $\sqrt{n}$ regret \citep{ADX10,HL14,Ito20,LZZ22}.
Finally, \cite{SRN21} study the adversarial problem where the loss function is (nearly) quadratic. They design 
a computationally efficient algorithm with $d^{16} \sqrt{n} \polylog(n)$ regret.
Like us, they also use Hessian estimates to control a focus region. Because they study the adversarial setting the situation
is more subtle. If the adversary dramatically changes the minimiser the algorithm needs to detect the change and restart or
broaden the focus region.  
The current state of affairs is given in \Cref{tab:sota}.

\begin{table}
{
\tiny
\renewcommand{\arraystretch}{1.9}
\newcommand{\gcell}{\cellcolor[gray]{0.9}}
\ifcolt
\else
\scalebox{1.08}{
\fi
\begin{tabular}{|p{4.2cm}cccccll|}
\hline
\tiny
\textbf{Authors} & \tiny\textbf{Constrained} & \tiny\textbf{Adversarial} & \tiny\textbf{Lipschitz} & \tiny\textbf{ Smooth} & \tiny\textbf{Strongly convex} & \tiny\textbf{Regret} $^\star$ & \tiny\textbf{Comp.} \\ \hline
\cite{FK05} & \checkmark & \checkmark &\gcell \checkmark & & & \gcell $\sqrt{d} n^{\frac{3}{4}}$ & $O(d)$ $^\dagger$ \\ \hline 
\cite{Sah11} & \checkmark & \checkmark & & \gcell \checkmark & & \gcell $\nu^{\frac{1}{3}} d^{\frac{2}{3}} n^{\frac{2}{3}}$ & $O(d)$ $^\dagger$ \\ \hline
\cite{HL14} & \checkmark & \checkmark & & \gcell \checkmark & \gcell \checkmark & $d \sqrt{\nu n}$ & $O(d)$ $^\dagger$ \\ \hline
\cite{BEL16} & \checkmark & \checkmark & & & & $d^{10.5} \sqrt{n}$ & \cellcolor[gray]{0.9}$\poly(d, n)$ \\ \hline
\cite{Lat20-cvx} & \checkmark & \checkmark & & & & $d^{2.5} \sqrt{n}$ & \cellcolor[gray]{0.9}$\exp(d, n)$ \\ \hline
\cite{LG21a} & \checkmark & \gcell & & & & $d^{4.5} \sqrt{n}$ & $O(d)$ $^\ddagger$ \\ \hline
\textbf{This work} & \gcell & \gcell & \gcell \checkmark & & & $d^{1.5} \sqrt{n}$ & $O(d^3)$ \\
\hline
\end{tabular} 
\ifcolt
\else
}
\fi
\vspace{-0.2cm}
\begin{flushleft}
\noindent $\star$ All regret bounds hold up to logarithmic factors for sufficiently large $n$ and omit any dependence on $r$ or the range of losses. The parameter $\nu$ is the self-concordance parameter for a barrier on the constraint set, which is $\nu = 1$ for the ball and information-theoretically never more than $d$. \,\,
 $\dagger$  These computation bounds assume the constraint set is a ball. 
 $\ddagger$ The algorithm uses the ellipsoid method and needs logarithmically many updates of $O(d^3)$ for the unconstrained case
 or $O(d^4)$ with a separation oracle on the constraint set. 
\end{flushleft}
\vspace{-0.5cm}
}
\caption{\small The current Pareto frontier for unconstrained zeroth-order bandit convex optimisation. 
Shaded cells correspond to poor behaviour of the corresponding algorithm in relation to the property associated with the cell.
Algorithms that do not depend on a Lipschitz assumption assume the loss is bounded on the constraint set.
The best lower bound is still $\Omega(d \sqrt{n})$ and uses linear functions \citep{DHK08}.
Algorithms for the constrained setting can also be used in the unconstrained one but not a-priori the other way.
}\label{tab:sota}
\end{table}

\paragraph{Notation}
The vector of all zeros is $\zeros$ and the identity matrix is $\id$, which will always be $d$-dimensional.
This should not be confused with the indicator function, denoted by $\sind(\cdot)$.
The density (with respect to Lebesgue) of the Gaussian distribution with mean $\mu$ and covariance $\Sigma$ is $\cN(\mu, \Sigma)$.
Given vector $x$ and square matrix $A$,
$\norm{x}$ is the standard euclidean norm and $\norm{x}_A^2 = x^\top A x$. 
The operator norm of a real matrix $A$ is 
$\norm{A} = \max_{x \neq \zeros} \norm{Ax} / \norm{x}$.
For positive semidefinite matrices $A$ and $B$ we write $A \leq B$ or $B \geq A$ to mean that $B - A$ is positive semidefinite.
For random elements $X$ and $Y$ taking values in the same space we write \smash{$X \stackrel{d}= Y$} if $\bbP(X \in A) = \bbP(Y \in A)$ for all measurable $A$. 
The complement of an event $E$ is $E^c$.
For a real random variable $W$ and $k \in \{1,2\}$, let
$\norm{W}_{\psi_k} = \inf\{t > 0 : \E[\exp(|W|^k/t^k)] \leq 2\}$.
A random variable is called subgaussian if $\norm{W}_{\psi_2} < \infty$ and subexponential if $\norm{W}_{\psi_1} < \infty$.
A simple corollary of the definitions is that $\norm{W^2}_{\psi_1} = \norm{W}_{\psi_2}^2$.

\paragraph{Constants}
The parameters of our algorithm are defined in terms of absolute constants. We let $\Cnst$ and $\cnst$ represent suitably
large/small absolute positive constants and $\cnstP = \max(2, n, d, \diam)^m$ where $m\geq 1$ is a suitably large absolute constant.
One can always check in the proofs that an appropriate choice of these constants is possible, by first choosing a small enough $\cnst$,
then a large enough $\Cnst$, and finally a large enough exponent $m$.

\paragraph{Sigma-algebras and noise sequence}
Let $\cF_t = \sigma(X_1,Y_1,\ldots,X_t,Y_t)$ be the $\sigma$-algebra generated by the action/loss sequence and
let $\E_t[\cdot] = \E[\cdot|\cF_t]$ and $\bbP_t(\cdot) = \bbP(\cdot | \cF_t)$.
Occasionally we need to specify the probability measure with respect to which an Orlicz norm is defined.
Given a random variable $X$ defined on some measurable space $(\Omega, \cF)$, we write $\norm{X}_{\bbP,\psi_k}$ for
the corresponding Orlicz norm with respect to probability measure $\bbP$ on $(\Omega, \cF)$. Very often $\bbP$ is $\bbP_t$ and for this
we make the abbreviation $\norm{X}_{t,\psi_k} = \norm{X}_{\bbP_t,\psi_k}$.
Our assumption on the noise sequence $(\epsilon_t)_{t=1}^n$ is that 
\begin{align}
\label{eq:subgauss}
\norm{\epsilon_t}_{\bbP_{t-1}(\cdot|X_t),\psi_2} \leq 1
\text{ and }
\E_{t-1}[\epsilon_t|X_t] = 0\,.
\end{align}
That is, after conditioning on the past action/losses and the current action, the noise is subgaussian and has mean zero.
If the bound on the Orlicz norm in \Cref{eq:subgauss} is replaced with the assumption that
$\norm{\epsilon_t}_{\bbP_{t-1}(\cdot|X_t),\psi_2} \leq \sigma$, then \Cref{thm:main} continues to hold with $1+\diam/d$
replaced with $\sigma+\diam/d$.

%%%%%%%%%%%%%%%%%%%%%%%%%%%%%%%%%%%%%%%%%%%%%%%%%%%%%%%%%%%%%
% ALGORITHM
%%%%%%%%%%%%%%%%%%%%%%%%%%%%%%%%%%%%%%%%%%%%%%%%%%%%%%%%%%%%%
\section{Algorithm}

Our algorithm is an instantiation of continuous exponential weights on the space of
Gaussian probability measures, which
was studied in the full information strongly convex setting by \cite{HEK18}. Our algorithm mirrors theirs except that
\textit{(a)} we make smaller Hessian updates to increase exploration; and \textit{(b)} we estimate gradients and Hessians
of a surrogate loss function rather than the actual loss function. What is nice about this method is that the unconstrained 
exponential weights
distribution for the special case of quadratic losses and Gaussian priors is also Gaussian with a straightforward update rule.
Remember that $\Cnst$ and $\cnst$ are suitably large/small absolute
constants and $\cnstP = \max(2, n, d, \diam)^m$ for a sufficiently large absolute constant $m$.
Let
\begin{align*}
\Wmax &= \Cnst \sqrt{d \log(\cnstP)} \,, &
\Dmax &= \Cnst \left(1 + \frac{\diam}{d}\right) \sqrt{\log(\cnstP)}\,, &
\eta &= \frac{\cnst}{\Dmax} \min\left(\sqrt{\frac{d}{n}},\,\, \frac{1}{d \sqrt{\log(\cnstP)}}\right)\,, \\
\Fmax &= \Cnst d^2 \log(\cnstP)^3\,, &
\lambda &= \frac{\cnst}{\sqrt{\Fmax \log(\cnstP)}} \,.
\end{align*}

\lstset{emph={to,for,then,end,for,if,input},emphstyle=\color{blue!100!black}\textbf}
\newcommand*{\Comment}[1]{\hfill\makebox[5.0cm][l]{\textbf{\texttt{\color{red!50!black}#1}}}}%
\begin{algorithm}
\begin{spacing}{1.3}
\begin{lstlisting}[mathescape=true,escapechar=\&,numbers=left,xleftmargin=5ex]
input $n$, $\diam$, $x_{\circ}$ 
let $\mu_1 = \mu_2 = x_{\circ}$ and $\Sigma_1 = \Sigma_2 = \frac{\diam^2}{d^2} \id$ 
sample $X_1$ from $\cN(\mu_1, \Sigma_1)$ and observe $Y_1 = f(X_1) + \epsilon_1$
for $t = 2$ to $n$
  sample and play $X_t$ from $\cN(\mu_t, \Sigma_t)$ and observe $Y_t = f(X_t) + \epsilon_t$
  compute $W_t = \Sigma_t^{-1/2} (X_t - \mu_t)$
  compute $T_t = \sind(|Y_t - Y_{t-1}| \leq \Dmax \text{ and } \norm{W_t} \leq \Wmax)$
  compute $D_t = T_t (Y_t - Y_{t-1})$
  $\gradt = D_t \Sigma_t^{-1} (X_t - \mu_t)$ &\Comment{\# gradient estimate}&
  $\hesst = \lambda D_t \Sigma_t^{-1/2} \left(W_t W_t^\top - \id \right) \Sigma_t^{-1/2}$ &\Comment{\# Hessian estimate}&
  update $\mu_{t+1} = \mu_t - \eta \Sigma_t \gradt$ and $\Sigma_{t+1}^{-1} = \Sigma_t^{-1} + \frac{1}{4} \eta \hesst$
end for
\end{lstlisting}
\end{spacing}
\caption{}\label{alg:cvx}
\end{algorithm}

Let us make some remarks on the unusual features of the algorithm as well as computation and parameter choices:
\begin{enumerate}[wide, labelwidth=!,labelindent=0pt]
\item The algorithm uses the loss differences $Y_t - Y_{t-1}$ between consecutive rounds. This is a variance reduction trick to replace the dependence on
the magnitude of the losses (on which we made no assumptions) to a dependence on the span of the losses over suitably sized balls. 
The latter is controlled using our assumption that the loss is Lipschitz.
\item The loss differences and the $W_t$ vectors are truncated if they are large, 
which ever so slightly biases the gradient and Hessian estimates.
Algorithmically this is unnecessary as we prove the truncation occurs with negligible probability. We leave it for convenience and because
it simplifies a little the analysis without impacting practical performance.
\item The Hessian estimate $H_t$ is symmetric but not positive definite.
Despite this, the choices of $\eta$, $\lambda$ and the truncation levels $\Dmax$ and $\Wmax$ ensure that $\Sigma_t$ (and its inverse)
remain positive definite.
\item 
The computational complexity is dominated by the eigendecomposition of $\Sigma_t$, which using practical methods is $O(d^3)$ floating point
operations.
The space complexity is $O(d^2)$.
\item
The theoretically justified recommendations for $\eta$ and $\lambda$ contain universal constants 
that we did not explicitly calculate. The reason is that the degree of the logarithmic term is excessively conservative.
We cautiously recommend dropping all the log factors and constants, which arise from (presumably) conservative high probability bounds.
This would give
\begin{align*}
\Wmax &= \infty\,, &
\Dmax &= \infty\,, &
\eta &= \frac{1}{1+\diam/d} \sqrt{\frac{d}{n}}\,, &
\lambda &= \frac{1}{d} \,. 
\end{align*}
Brief experiments suggest the algorithm remains stable with these choices but if the algorithm eventually becomes useful,
then either the theory can be fine-combed to optimise the constants or better choices can be found empirically.
Even better would be to find a crisper analysis that does not rely on an inductive high probability argument.
\end{enumerate}

%%%%%%%%%%%%%%%%%%%%%%%%%%%%%%%%%%%%%%%%%%%%%%%%%%%%%%%%
% SURROGATE
%%%%%%%%%%%%%%%%%%%%%%%%%%%%%%%%%%%%%%%%%%%%%%%%%%%%%%%%
\section{Surrogate loss function}

We start by reintroducing the surrogate loss function used by \cite{BEL16} and \cite{LG21a}.
Let $\mu \in \R^d$ and $\Sigma \in \R^{d\times d}$ be positive definite and $p = \cN(\mu, \Sigma)$. Given $\lambda \in (0,1)$, define
\begin{align*}
s(z) 
= \int_{\R^d} \left[\left(1 - \frac{1}{\lambda}\right) f(x) + \frac{1}{\lambda} f((1 - \lambda) x + \lambda z)\right] p(x) \d{x}\,.
\end{align*}
\cite{LG21a} noted that $s$ is convex and $s(x) \leq f(x)$ for all $x$, both of which follow almost immediately 
from convexity of $f$ (\Cref{fig:g}). Because $s$ is defined by a convolution with a Gaussian and $f$ is Lipschitz, $s$ is also infinitely differentiable.
In general, $s$ is a good approximation of $f$ when the latter is close to linear and a poor approximation when $f$ has considerable
curvature.
The next lemma collects a variety of properties of the surrogate loss 
(proof in \Cref{sec:lem:equal}).

\begin{lemma}\label{lem:equal}
Suppose that $Z$ has law $q = \cN(\mu, \beta^2 \Sigma)$ with $\beta^2 = (2 - \lambda)/\lambda$ and $X$ has law $p$.
Then 
\begin{enumerate}[nosep]
\item $s$ is convex and $s(z) \leq f(z)$ for all $z \in \R^d$\,;
\item $\E[s(Z)] = \E[f(X)]$\,; 
\item $\E[\nabla s(Z)] = \E[f(X) \Sigma^{-1}(X - \mu)]$\,; \label{eq:it:sgrad}
\item $\E[\nabla^2 s(Z)] = \lambda \E[f(X)\Sigma^{-1}((X - \mu)(X - \mu)^\top \Sigma^{-1} - \id)]$\,; \label{eq:it:sHess}
\item $\norm{\nabla^2 s(z)} \leq \norm{\Sigma^{-1/2}}$ for all $z \in \R^d$\,;
\item Suppose $\smash{z, w \in \R^d}$ and $\smash{\epsilon = \frac{\lambda}{1 - \lambda} (z - w)}$ satisfies
$\smash{\norm{\epsilon}^2_{\Sigma^{-1}} \leq \frac{\log(2)^2}{2 \log(\cnstP)}}$.
Then
\begin{align*}
\nabla^2 s(z) \leq 2 \nabla^2 s(w) + \frac{\norm{\Sigma^{-1/2}}}{\cnstP} \id \,.
\end{align*}
\item $\E[\ip{\nabla s(Z), Z - \mu}] = \beta^2 \E[\tr(\Sigma \nabla^2 s(Z))]$.
\end{enumerate}
\end{lemma}
Note that if $f$ is twice differentiable, then \ref{eq:it:sgrad} equals $\E[\nabla f(X)]$ and \ref{eq:it:sHess}
equals $\E[\nabla^2 f(X)]$.

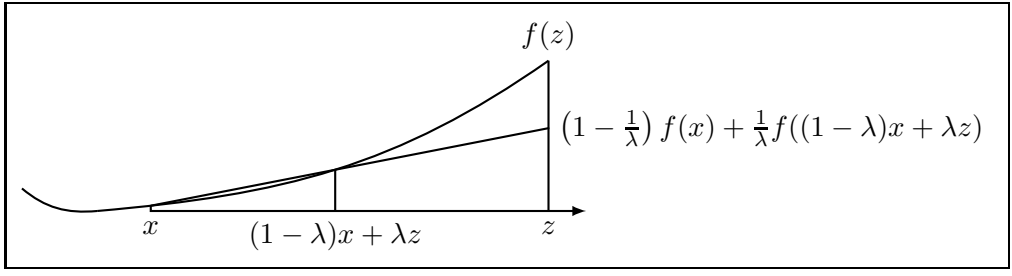
\begin{figure}[th]
\centering
\fbox{
\begin{tikzpicture}
\draw[thick] (0,0) to[out=5,in=-145] coordinate[pos=0.1] (a) coordinate[pos=0.5] (b) (6,2);
\draw[thick] (0,0) to[out=185,in=-40] (-1,0.3);
\draw[thick,shorten >=-2.9cm] (a) -- (b);
\draw[thick,-latex] (a) -- (a |- 0,0) -- (6.5,0); 
\draw[thick] (6,2) -- (6,0);
\draw[thick] (b) -- (b |- 0,0); 
\node[below] at (a |- 0,0) {$x$};
\node[below] at (b |- 0,0) {$(1 - \lambda) x + \lambda z$};
\node[below] at (6,0) {$z$};
\node[above] at (6,2) {$f(z)$};
\node[right] at (6,1.1) {$\left(1 - \frac{1}{\lambda}\right) f(x) + \frac{1}{\lambda} f((1 - \lambda) x + \lambda z)$};
\end{tikzpicture}
}
\caption{Given a fixed $z$ and $x$, let $y = (1 - \lambda) x + \lambda z$. A lower bound of $f(z)$
can be found by evaluating the second coordinate of the linear function through $(x, f(x))$ and $(y, f(y))$ at $z$,
which is $(1-\frac{1}{\lambda}) f(x) + \frac{1}{\lambda} f(y)$.
Then $s(z)$ is the average of this value over all $x$ when $x$ has law $p$.}
\label{fig:g}
\end{figure}

%%%%%%%%%%%%%%%%%%%%%%%%%%%%%%%%%%%%%%%%%%%%%%%%%%%%%%%%%%%%%%%%
% PROOF OF MAIN THEOREM
%%%%%%%%%%%%%%%%%%%%%%%%%%%%%%%%%%%%%%%%%%%%%%%%%%%%%%%%%%%%%%%%
\section{Proof of Theorem~\ref{thm:main}}
Assume without loss of generality that $x_\star = \zeros$, which means that the initialisation 
of the algorithm $x_{\circ}$ satisfies $\norm{x_\circ} \leq \diam$.
Let $\smash{\normt{t}{\cdot} = \norm{\cdot}_{\Sigma_t^{-1}}}$.
At a high level the analysis follows the classical analysis of mirror descent.
The main conceptual challenge is proving that with high probability 
\begin{align}
\frac{1}{2} \normt{t+1}{\mu_{t+1}}^2 \leq \Fmax - \eta \sum_{s=1}^t \Delta_s
\label{eq:potential}
\end{align}
holds for all $t$,
where $$\Delta_t = \E_{t-1}[f(X_t) - f(x_\star)]$$ is the instantaneous expected regret.
Rearranging \Cref{eq:potential} with $t = n$ yields a bound on the regret. Just as important, however, is that \Cref{eq:potential}
ensures that the optimal point $x_\star = \zeros$ lies 
in the \textit{focus region} $\{\nu \in \R^d : \frac{1}{2} \normt{t}{\nu - \mu_t}^2 \leq \Fmax\}$, which is the
region in which the surrogate loss function behaves more-or-less like a quadratic. Essentially we prove \Cref{eq:potential} holds
with high probability by induction, using in the inductive step that the optimal point lies in the focus region and hence the estimator
is well-behaved. 

\begin{definition}
\label{def:E}
Let $E_t$ be the event that
\begin{enumerate}
\item $\Sigma_{t+1} \leq 2 \Sigma_1$\,;
\item $\tr\left(\Sigma_{t+1}^{-1}\right) \leq \Sigma_{\max}^{-1} \triangleq \left(\frac{n d^2}{\diam^2} + \frac{dn^2}{4} + n\right)^2$\,; 
\item $|\E_t[Y_{t+1}] - Y_t| \leq \frac{1}{2}\Dmax$\,.
\end{enumerate}
Define a stopping time $\tau$ as the first round $t$ where either $E_t$ does not hold or
\begin{align*}
\frac{1}{2} \normt{t+1}{\mu_{t+1}}^2 \geq \Fmax - \eta \sum_{s=1}^t \Delta_s \,.
\end{align*}
If neither condition ever holds, then $\tau$ is defined to be $n$.
\end{definition}

Note that $\Sigma_{t+1}$ and $\mu_{t+1}$ are both $\cF_t$-measurable, so $\tau$ really is a stopping 
time with respect to the filtration $(\cF_t)_{t=1}^n$.
Properties \eref{(a)} and \eref{(b)} indicate that neither $\Sigma_{t+1}$ nor 
its inverse grows too large. Property \eref{(c)} indicates that the losses do not change dramatically from one round to the next.
In other words, properties \eref{(a)}-\eref{(c)} are indicators that the algorithm is stable.
The following lemma shows that the algorithm is stable with high probability.

\begin{lemma}\label{lem:stab}
$\bbP(\cap_{t=1}^\tau E_t) \geq 1 - 4/n$.
\end{lemma}

Let $\beta^2 = (2-\lambda)/\lambda$ and
$Z_t$ be a random variable that is independent of $X_t$ and under $\bbP_{t-1}$ has law $\cN(\mu_t, \beta^2 \Sigma_t)$, and define the surrogate loss at time $t$ as
\begin{align*}
s_t(z) = \E_{t-1}\left[\left(1 - \frac{1}{\lambda}\right) f(X_t) + \frac{1}{\lambda} f((1 - \lambda)X_t + \lambda z)\right]  \,.
\end{align*}
The truncation in the gradient and Hessian estimators introduces a small amount of bias that needs to be controlled (proof in \Cref{app:lem:bias}).

\begin{lemma}\label{lem:bias}
On $\{t \leq \tau\}$ and for a positive definite matrix $A$,
\begin{enumerate}
\item $\displaystyle \Big|\E_{t-1}\left[\ip{\gradt, \mu_t}\right] - \E_{t-1}\left[\ip{\nabla s_t(Z_t), \mu_t}\right]\Big| \le \frac{\lambda}{n}$\,;
\item $\displaystyle \Big|\E_{t-1}\left[\tr(A \hesst) \right] - \E_{t-1}\left[\tr(A \nabla^2 s_t(Z_t))\right]\Big| 
\leq \tr(A \Sigma_t^{-1}) \min\left(\frac{\lambda}{2n \Fmax},\, \frac{\lambda}{dn},\, \frac{1}{n \Sigma_{\max}^{-1}}\right)$\,.
\end{enumerate}
\end{lemma}

\paragraph{Step 1: High-level argument}
Expanding the square shows that
\begin{align*}
\frac{1}{2} \normt{t+1}{\mu_{t+1}}^2 - \frac{1}{2} \normt{t}{\mu_t}^2
&= \underbracket{-\eta \ip{\gradt, \mu_t} + \frac{\eta}{8} \norm{\mu_t}^2_{\hesst}}_{\A_t} 
  + \underbracket{\frac{\eta^2}{2}\norm{\gradt}^2_{\Sigma_t} + \frac{\eta^3}{8} \norm{\Sigma_t \gradt}^2_{\hesst} - \frac{\eta^2}{4} \ip{\mu, \Sigma_t \gradt}_{\hesst}}_{\B_t} \,.
\end{align*}
$\A_t$ collects those terms that are linear in the learning rate and $\B_t$ those that are quadratic or cubic. 
The expectation of the linear term will be shown to be close to $-\eta \Delta_t$ (recall that $\Delta_t = \E_{t-1}[f(X)] - f(\zeros)$ 
is the expected instantaneous regret).
The lower order terms will be shown to be $O(\eta^2)$. Besides technical complications, the result follows by dividing both sides by the learning rate, rearranging and telescoping the potentials.
The principle difficulty is that our bounds on $\A_t$ and $\B_t$ only hold when 
$\frac{1}{2} \normt{t}{\mu_t}^2$ 
is not too large, which has to be tracked through the analysis with induction
and a high probability argument.

%%%%%%%%%%%%%%%%%%%%%%%%%%%%%%%%%%%%%%%%%%%%%%%%%%%%%%
% STEP 3: LINEAR TERMS
%%%%%%%%%%%%%%%%%%%%%%%%%%%%%%%%%%%%%%%%%%%%%%%%%%%%%%
\paragraph{Step 2: Linear terms}
This step is the most fundamental. We show that the linear terms can
be controlled in terms of the regret and some small correction terms.

\begin{lemma}\label{lem:A}
With probability at least $1 - 1/n$,
\begin{align*}
\sum_{t=1}^\tau \A_t 
\leq -\eta \sum_{t=1}^\tau \Delta_t + \eta\beta^2 \sum_{t=1}^\tau \tr(\Sigma_t\hesst) + 1100 \eta \Dmax \Fmax^{1/2} \sqrt{n \log(\cnstP)} \,.
\end{align*}
\end{lemma}

\begin{proof}
Suppose that $t \leq \tau$.
Remember that $Z_t$ is a random element that under $\bbP_{t-1}$ has law $\cN(\mu_t, \beta^2 \Sigma_t)$ and is independent from $X_t$.
Then
\begin{align*}
&\E_{t-1}[\A_t]
= \E_{t-1}\left[-\eta \ip{\gradt, \mu_t} + \frac{\eta}{8} \norm{\mu_t}^2_{\hesst}\right] \\
\tag*{Lemma~\ref{lem:bias}\er{ab}}
&\quad\leq \E_{t-1}\left[-\eta \ip{\nabla s_t(Z_t), \mu_t} + \frac{\eta}{8} \norm{\mu_t}^2_{\nabla^2 s_t(Z_t)}\right] + \frac{2 \eta \lambda}{n} \\
&\quad= \E_{t-1}\left[-\eta \ip{\nabla s_t(Z_t), Z_t} + \eta \ip{\nabla s_t(Z_t), Z_t - \mu_t} + \frac{\eta}{8} \norm{\mu_t}^2_{\nabla^2 s_t(Z_t)}\right] + \frac{2\eta\lambda}{n} \\
\tag*{\Cref{lem:equal}\eref{g}}
&\quad= \E_{t-1}\left[-\eta \ip{\nabla s_t(Z_t), Z_t} + \eta \beta^2 \tr(\Sigma_t \nabla^2 s_t(Z_t)) + \frac{\eta}{8} \norm{\mu_t}^2_{\nabla^2 s_t(Z_t)}\right] + \frac{2\eta\lambda}{n} \\
\tag*{Lemma~\ref{lem:bias}\er{b}}
&\quad\leq \E_{t-1}\left[-\eta \ip{\nabla s_t(Z_t), Z_t} + \eta \beta^2 \tr(\Sigma_t \hesst) + \frac{\eta}{8} \norm{\mu_t}^2_{\nabla^2 s_t(Z_t)}\right] + \frac{4\eta}{n}  \,,
\end{align*}
where in the first inequality we used that $\tr(\mu\mu^\top \Sigma_t^{-1}) \leq 2\Fmax$ since $t \le \tau$ and
in the second we used $\lambda \leq 1$ and $\lambda \beta^2 = 2-\lambda \leq 2$.
To make progress on bounding the first term, recall that $s_t$ is infinitely differentiable. 
Hence, by Taylor's theorem, for all $z \in \R^d$ 
there exists a $\xi_z \in [\zeros, z] = \{ \alpha z : \alpha \in [0,1]\}$ such that
\begin{align*}
s_t(\zeros) = s_t(z) - \ip{\nabla s_t(z), z} + \frac{1}{2} \norm{z}^2_{\nabla^2 s_t(\xi_z)} \,.
\end{align*}
We need a simple lemma (proof in \Cref{app:lem:hess2} based on Lemma~\ref{lem:equal}\er{f}) to bound $\nabla^2 s_t(\xi_z)$.

\begin{lemma}\label{lem:hess2}
On $\{t \le \tau\}$,
\begin{enumerate}
\item $\E_{t-1}\left[\norm{Z_t}^2_{\nabla^2 s_t(\xi_{Z_t})}\right] \geq \frac{1}{2} \E_{t-1}\left[\norm{Z_t}^2_{\nabla^2 s_t(\mu_t)}\right] 
  - \frac{1}{n}$.
\item $\E_{t-1}\left[\norm{\mu_t}^2_{\nabla^2 s_t(\mu_t)}\right] \geq \frac{1}{2} \E_{t-1}\left[\norm{\mu_t}^2_{\nabla^2 s_t(Z_t)}\right] 
- \frac{1}{n}$.
\end{enumerate}
\end{lemma}

Using Lemma~\ref{lem:hess2},
\begin{align*}
\E_{t-1}[\ip{\nabla s_t(Z_t), Z_t}] 
&= \E_{t-1}\left[s_t(Z_t) - s_t(\zeros) + \frac{1}{2} \norm{Z_t}^2_{\nabla^2 s_t(\xi_{Z_t})}\right] \\
\tag*{Lemma~\ref{lem:hess2}\er{a}}
&\geq \E_{t-1}\left[s_t(Z_t) - s_t(\zeros) + \frac{1}{4} \norm{Z_t}^2_{\nabla^2 s_t(\mu_t)}\right] - \frac{1}{n} \\
&= \E_{t-1}\left[s_t(Z_t) - s_t(\zeros) + \frac{\beta^2}{4} \tr(\Sigma_t \nabla^2 s_t(\mu_t)) + \frac{1}{4} \norm{\mu_t}^2_{\nabla^2 s_t(\mu_t)}\right] - \frac{1}{n} \\
\tag*{Lemma~\ref{lem:hess2}\er{b}} 
&\geq \E_{t-1}\left[s_t(Z_t) - s_t(\zeros) + \frac{1}{8} \norm{\mu_t}^2_{\nabla^2 s_t(Z_t)}\right] - \frac{2}{n} \\
\tag*{Lemma~\ref{lem:equal}\er{ab}}
&\geq \E_{t-1}\left[f(X_t) - f(\zeros) + \frac{1}{8} \norm{\mu_t}^2_{\nabla^2 s_t(Z_t)}\right] - \frac{2}{n} \\
\tag*{Definition of $\Delta_t$}
&= \Delta_t + \frac{1}{8}\E_{t-1}\left[\norm{\mu_t}^2_{\nabla^2 s_t(Z_t)}\right] - \frac{2}{n} \,,
\end{align*}
where in the second equality we used the fact that $Z_t$ has law $\cN(\mu_t, \beta^2 \Sigma_t)$ under $\bbP_{t-1}$, 
and in the second inequality that the matrices in the dropped term are positive semi-definite.
Hence, 
\begin{align*}
\E_{t-1}[\A_t] \leq -\eta \Delta_t + \eta \beta^2 \E_{t-1}[\tr(\Sigma_t \hesst)] + \frac{6 \eta}{n}\,.
\end{align*}
This shows the connection between the linear component of the change in the potential and the regret.
The remainder of the proof of the lemma is devoted to a concentration analysis converting the bound in expectation to something that holds with high probability.
By the above display and the definition of $\A_t$,
\begin{align*}
\sum_{t=1}^\tau \A_t
&\leq \sum_{t=1}^\tau \left[\A_t - \E_{t-1}[\A_t]\right] 
    + \eta \beta^2 \sum_{t=1}^\tau \big(\E_{t-1}[\tr(\Sigma_t \hesst)] - \tr(\Sigma_t \hesst)\big) \\
&\qquad\qquad    + \eta \beta^2 \sum_{t=1}^\tau \tr(\Sigma_t \hesst)
   - \eta \sum_{t=1}^\tau \Delta_t + 6\eta \,.
\end{align*}
The first two terms on the right-hand side are sums of martingale differences, which we now control using concentration of measure.
We need to show that the tails of $A_t$ and $\tr(\Sigma_t \hesst)$ are well-behaved under $\bbP_{t-1}$ whenever $t \leq \tau$.
Assume that $t \leq \tau$.
Then, since $D_t \leq \Dmax$ and $\frac{1}{2} \normt{t}{\mu_t}^2 \leq \Fmax$, Fact~\ref{fact:sg}\er{a} implies
\begin{align*}
\norm{\eta \ip{\mu_t, \gradt}}_{t-1,\psi_2}
= \norm{\eta D_t \ip{\Sigma_t^{-1/2}\mu_t, W_t}}_{t-1,\psi_2}
\leq 2 \eta \Dmax \normt{t}{\mu_t}
\leq 3 \eta \Dmax \Fmax^{1/2}\,.
\end{align*}
Lemma 2.7.7 in the book by \cite{Ver18} says that $\norm{XY}_{\psi_1} \leq \norm{X}_{\psi_2} \norm{Y}_{\psi_2}$ for any random variables 
$X$ and $Y$.
By definition $\norm{1}_{\psi_2} = 1/\sqrt{\log(2)}$. Combining these with the above display shows that
\begin{align*}
\norm{\eta \ip{\mu_t, \gradt}}_{t-1,\psi_1} \leq 4 \eta \Dmax \Fmax^{1/2}\,.
\end{align*}
Next, using Fact~\ref{fact:sg}\er{b}, $\norm{1}_{\psi_1}=1/\log(2)$, and that $\lambda \leq \Fmax^{-1/2}$,
\begin{align*}
\bnorm{\frac{\eta}{8} \norm{\mu_t}^2_{\hesst}}_{t-1,\psi_1}
&= \bnorm{\frac{\eta\lambda D_t}{8} \norm{\Sigma_t^{-1/2} \mu_t}^2_{W_tW_t^\top - \id}}_{t-1,\psi_1}  
\leq \eta \lambda \Dmax \normt{t}{\mu_t}^2 
\leq 2\eta \Dmax \Fmax^{1/2}\,.
\end{align*}
Combining the above two displays with the triangle inequality implies that $\norm{A_t}_{t-1,\psi_1} \le 6 \eta \Dmax \Fmax^{1/2}$.
Lastly,  since $\lambda \beta^2 = 2 - \lambda \leq 2$ and $d \leq \Fmax^{1/2}$\,,
\begin{align*}
\bnorm{\eta \beta^2 \tr(\Sigma_t \hesst)}_{t-1,\psi_1}
&= \eta \lambda \beta^2 \bnorm{D_t \tr(W_tW_t^\top - \id)}_{t-1,\psi_1} 
\leq 5 \eta \lambda \beta^2 d \Dmax
\leq 10 \eta \Dmax \Fmax^{1/2} \,.
\end{align*}
The claim of the lemma now follows by Bernstein's inequality (Lemma~\ref{lem:bernstein}) and naively simplifying the constants.
\end{proof}

%%%%%%%%%%%%%%%%%%%%%%%%%%%%%%%%%%%%%%%%%%%%%%%%%%%%%%
% STEP 4: QUADRATIC TERMS
%%%%%%%%%%%%%%%%%%%%%%%%%%%%%%%%%%%%%%%%%%%%%%%%%%%%%%
\paragraph{Step 3: Lower-order terms}
The next step is to bound the lower-order terms with high probability.

\begin{lemma}\label{lem:B}
$\bbP\left(\sum_{t=1}^\tau \B_t \geq 31 nd  \eta^2  \Dmax^2 \log(\cnstP)\right) \leq 1/n$.
\end{lemma}

\begin{proof}
We show that on $\{t \leq \tau\}$,
\begin{align}
\bbP_{t-1}\left(\B_t \geq 31 d \eta^2 \Dmax^2 \log(\cnstP)\right) \leq \frac{1}{n^2} \,.
\label{eq:B-cond}
\end{align}
The lemma then follows by a union bound.
The remainder of the proof is devoted to establishing \Cref{eq:B-cond}.
To this end, suppose $t \leq \tau$.
We start by bounding $B_t$ using Cauchy-Schwarz and Young's inequality:
\begin{align}
\B_t
&= \frac{\eta^2}{2} \norm{\gradt}^2_{\Sigma_t} + \frac{\eta^3}{8} \norm{\Sigma_t \gradt}^2_{\hesst} 
  - \frac{\eta^2}{4} \ip{\mu_t, \Sigma_t \gradt}_{\hesst} \nonumber \\
&\leq \frac{\eta^2}{2} \norm{\gradt}^2_{\Sigma_t} + \frac{\eta^3}{8} \norm{\Sigma_t \gradt}^2_{\hesst} 
  + \eta^2 \norm{\mu_t}_{\hesst \Sigma_t \hesst} \norm{\gradt}_{\Sigma_t} \nonumber \\
&\leq \eta^2 \norm{\gradt}^2_{\Sigma_t} + \frac{\eta^3}{8} \norm{\Sigma_t \gradt}^2_{\hesst} 
+ \frac{\eta^2}{2} \norm{\mu_t}^2_{\hesst \Sigma_t \hesst} \,.
\label{eq:B}
\end{align}
By Fact~\ref{fact:sg} and Lemma~\ref{lem:sg}\er{ab}, with $\bbP_{t-1}$-probability at least $1 - 1/n^2$ the following hold:
\begin{enumerate}
\item $\norm{W_t}^2 \leq 3d \log(6n^2)$\,;
\item $\norm{W_t W_t^\top - \id} \leq 5d \log(6n^2)$\,;
\item $\ip{\Sigma_t^{-1/2} \mu_t, W_t}^2 \leq 4 \normt{t}{\mu_t}^2 \log(6n^2)$.
\end{enumerate}
The next step is to bound each of the three terms on the right-hand side of \Cref{eq:B} under the assumption that \eref{(a)-(c)} above hold.
Using the definition of $\gradt$,
\begin{align*}
\eta^2 \norm{\gradt}^2_{\Sigma_t}
= \eta^2 D_t^2 \norm{W_t}^2  
\leq 3\eta^2 \Dmax^2 d \log(6n^2)
\leq 3\eta^2 \Dmax^2 d \log(\cnstP)\,,
\end{align*}
where in the final step we made sure to choose $\cnstP$ large enough.
Moving on, 
\begin{align*}
\frac{\eta^3}{8} \norm{\Sigma_t \gradt}^2_{\hesst} 
&\leq \frac{\eta^3}{8} \norm{\gradt}^2_{\Sigma_t} \norm{\Sigma_t^{1/2} \hesst \Sigma_t^{1/2}} \\
&= \frac{\lambda \eta^3 |D_t|^3}{8} \norm{W_t}^2 \norm{W_tW_t^\top - \id} \\
&\leq 2\lambda \eta^3 \Dmax^3 d^2 (\log(6n^2))^2 \\
&\leq 2 \eta^2 \Dmax^2 d \log(\cnstP)\,,
\end{align*}
where in the final inequality we used the fact that $\eta \lambda \leq [d \Dmax \log(\cnstP)]^{-1}$ and chose $\cnstP$ suitably
large.
Finally, a calculation shows that
\begin{align*}
\hesst \Sigma_t \hesst 
&= \lambda^2 D_t^2 \Sigma_t^{-1/2} (W_t W_t^\top - \id)(W_t W_t^\top - \id) \Sigma_t^{-1/2}  \\
&= \lambda^2 D_t^2 \Sigma_t^{-1} ((X_t - \mu_t)(X_t - \mu_t)^\top \Sigma_t^{-1} - \id)((X_t - \mu_t)(X_t - \mu_t)^\top \Sigma_t^{-1} - \id)  \\
&= \lambda^2 D_t^2 \left[\norm{W_t}^2 \Sigma_t^{-1/2} W_tW_t^\top \Sigma_t^{-1/2} + \Sigma_t^{-1} - 2 \Sigma_t^{-1/2} W_tW_t^\top \Sigma_t^{-1/2} \right] \\
&\leq \lambda^2 D_t^2 \left[\norm{W_t}^2 \Sigma_t^{-1/2} W_tW_t^\top \Sigma_t^{-1/2} + \Sigma_t^{-1}\right]\,.
\end{align*}
Therefore, remembering that on $\{t \leq \tau\}$, $\frac{1}{2}\normt{t}{\mu_t}^2 \leq \Fmax$,
\begin{align*}
\eta^2\norm{\mu_t}^2_{\hesst \Sigma_t \hesst} 
&\leq \eta^2 \lambda^2 D_t^2 \left[\norm{W_t}^2 \ip{\Sigma_t^{-1/2} \mu_t, W_t}^2 + \norm{\mu_t}^2_{\Sigma_t^{-1}}\right] \\
&\leq 13\eta^2 \lambda^2 d \Dmax^2 \normt{t}{\mu_t}^2 (\log(6n^2))^2 \\
&\leq 26\eta^2 \lambda^2 d\Dmax^2 \Fmax (\log(6n^2))^2 \\
\tag*{$\lambda^2 \leq [\Fmax \log(\cnstP)]^{-1}$}
&\leq 26\eta^2 \Dmax^2 d \log(\cnstP)\,.
\end{align*}
Combining everything establishes \Cref{eq:B-cond} and so too the lemma.
\end{proof}

%%%%%%%%%%%%%%%%%%%%%%%%%%%%%%%%%%%%%%%%%%%%%%%%%%%%%%
% STEP 4: REGRET
%%%%%%%%%%%%%%%%%%%%%%%%%%%%%%%%%%%%%%%%%%%%%%%%%%%%%%
\paragraph{Step 4: Bounding the regret}
By Lemma~\ref{lem:stab}, Lemma~\ref{lem:A} and Lemma~\ref{lem:B}, with probability least $1 - 6/n$, $\bigcap_{t=1}^\tau E_t$ holds and
\begin{align*}
&\sum_{t=1}^\tau \frac{1}{2} \normt{t+1}{\mu_{t+1}}^2 - \frac{1}{2} \normt{t}{\mu_t}^2
= \sum_{t=1}^\tau (\A_t + \B_t) \\
&\qquad\leq 1100 \eta \Dmax \Fmax^{1/2} \sqrt{n \log(\cnstP)} + 31nd\eta^2 \Dmax^2 \log(\cnstP) + \eta \beta^2 \sum_{t=1}^\tau \tr(\Sigma_t \hesst) - \eta\sum_{t=1}^\tau \Delta_t \\
&\qquad\leq 1100 c \Fmax^{1/2} \sqrt{d \log(\cnstP)} + 31 c^2 d^2 \log(\cnstP) + \eta \beta^2 \sum_{t=1}^\tau \tr(\Sigma_t \hesst) - \eta \sum_{t=1}^\tau \Delta_t\,,
\end{align*}
where the second equality follows from the definition of $\eta$.
Using that $x \leq 2\log(1+x)$ for $x \in [0,1]$ it follows that for positive definite $X$ with $\norm{X} \leq 1$,
$\tr(X) \leq 2\log \det(\id + X)$. Therefore, on the event $\cap_{t=1}^\tau E_t$, %by Definition~\ref{def:E}\er{a}, 
\begin{align*}
\eta \beta^2 \sum_{t=1}^{\tau} \tr(\Sigma_t \hesst)
&\leq 8\beta^2 \sum_{t=1}^{\tau} \log \det\left(\id + \frac{\eta}{4} \Sigma_t \hesst\right) \\
&= 8\beta^2 \log\left(\det\left(\Sigma_1\Sigma_{\tau+1}^{-1}\right)\right) \\
\tag*{Jensen's inequality}
&\leq \frac{16d}{\lambda} \log\left(\frac{\tr\left(\Sigma_1 \Sigma_{\tau+1}^{-1}\right)}{d}\right) \\
&\leq \frac{16d \Fmax^{1/2} \log(\cnstP)^{3/2}}{c}\,,
\end{align*}
where the equality holds because $\Sigma_t \Sigma_{t+1}^{-1} = \id + \frac{\eta}{4} \Sigma_t \hesst$ by the update for $\Sigma_{t+1}^{-1}$, in the second inequality we used $\beta^2 = (2-\lambda)/\lambda \leq 2/\lambda$, and the last inequality holds because $\tr(\Sigma_{\tau+1}^{-1}) \leq \Sigma_{\max}^{-1}$ by Definition~\ref{def:E}\er{b}, the definition of $\Sigma_1$, and by choosing $\cnstP$ large enough.
Therefore with probability at least $1 - 6/n$, $E_\tau$ holds and 
\begin{align*}
\frac{1}{2} \normt{\tau+1}{\mu_{\tau+1}}^2
&\leq \frac{1}{2} \normt{1}{\mu_1}^2 + \frac{16d \Fmax^{1/2}}{c} \log(\cnstP)^{3/2} + 1100c \Fmax^{1/2} \sqrt{d \log(\cnstP)} + 31c^2 d^2 \log(\cnstP) - \eta \sum_{t=1}^\tau \Delta_t \\
&< \Fmax - \eta \sum_{t=1}^\tau \Delta_t\,,
\end{align*}
where in the second inequality we used the definition of $\Fmax$ and by choosing $\Cnst$ 
suitably large and the fact that $\frac{1}{2} \normt{1}{\mu_1}^2 \leq d^2/2$.
Since $\{\frac{1}{2} \normt{\tau+1}{\mu_{\tau+1}}^2 \leq \Fmax- \eta \sum_{t=1}^\tau \Delta_t\} \cap E_\tau$ implies that $\tau=n$, it follows by rearranging the above display that
\begin{align*}
\bbP\left(\sum_{t=1}^n \Delta_t \le \frac{\Fmax}{\eta} \quad \text{and} \quad \tau=n\right) \ge 1-\frac{6}{n}\,.
\end{align*}
The last step is to bound the actual regret in terms $\sum_{t=1}^n \Delta_t$.
By Lemma~\ref{lem:lip-conc}\er{a}, on $\{t \leq \tau\}$,
\begin{align*}
\norm{f(X_t) - \E_{t-1}[f(X_t)]}_{t-1,\psi_2} 
&\leq 2 \norm{\Sigma_t}^{1/2}
\leq \frac{3 \diam}{d}\,.
\end{align*}
Therefore by Lemma~\ref{lem:bernstein}, with probability at least $1 - 1/n$,
\begin{align*}
\sum_{t=1}^\tau f(X_t) - f(\zeros) \leq \sum_{t=1}^\tau \Delta_t + \frac{200 \diam}{d} \sqrt{n \log(n)}\,.
\end{align*}
Hence, with probability at least $1-7/n$,
\begin{align*}
\sum_{t=1}^n f(X_t) - f(\zeros) \le \frac{\Fmax}{\eta} + \frac{200 \diam}{d} \sqrt{n \log(n)}\,.
\end{align*}
\Cref{thm:main} now follows from the definitions of $\eta$ and $\Fmax$.

%%%%%%%%%%%%%%%%%%%%%%%%%%%%%%%%%%%%%%%%%%%%%%%%%%%%%%%%%%%%%%%%%%%
% DISCUSSION
%%%%%%%%%%%%%%%%%%%%%%%%%%%%%%%%%%%%%%%%%%%%%%%%%%%%%%%%%%%%%%%%%%%
\section{Discussion}

There are a few outstanding issues.

\paragraph{Handling constraints}
Our algorithm cannot handle constraints on the domain of $f$.
There are several ideas one may try. For example, by estimating some kind of extension of $f$ or regularising to prevent
the focus region from leaving the domain. It would surprise us if nothing can be made to work, possibly at the price of worse dimension-dependence.

\paragraph{Adversarial setting}
Algorithms based on elimination or focus regions cannot handle the adversarial setting without some sort of correction.
\cite{BEL16} and \cite{SRN21} both use restarts, which may also be usable in our setting.
Note that in the adversarial version of the problem the centering of the gradient/Hessian estimators using the loss 
from the previous round no longer makes sense and the dependence on $d$ in front of the diameter 
should be expected to increase slightly.

\paragraph{Dependence on various quantities}
A natural question is whether or not there is scope to improve the bound. With these techniques, it feels like there is limited room
for improvement. In particular the bounds on the stability and variance of the algorithm seem to be tight.
There is \textit{still} no lower bound that is superlinear in the dimension.
Maybe the true dimension dependence is linear in $d$, but fundamentally new ideas seem to be needed for such a result.
One may also wonder about the dependence on $\diam$. The quantity $\diam/d$ is effectively the range of the
observed losses. Because our setting is unconstrained, we cannot assume the losses
are globally bounded in $[0,1]$ as is standard in the constrained setting. Our expectation is that once the analysis is applied to the constrained
case, the quantity $\diam/d$ will be replaced by $1$.

\paragraph{Sample complexity}
\Cref{cor:main} shows that $\frac{1}{n} \sum_{t=1}^n X_t$ is near-optimal with high probability for suitably large $n$.
Using convexity one can easily show that $\frac{1}{n} \sum_{t=1}^n \mu_t$ is also near-optimal with
the same sample complexity and is unsurprisingly empirically superior.

\appendix
\crefalias{section}{appendix}

\ifcolt
\else
\bibliographystyle{plainnat}
\fi
\bibliography{all}

%%%%%%%%%%%%%%%%%%%%%%%%%%%%%%%%%%%%%%%%%%%%%%%%%%
% EQUAL LEMMA
%%%%%%%%%%%%%%%%%%%%%%%%%%%%%%%%%%%%%%%%%%%%%%%%%%
\section{Proof of Lemma~\ref{lem:equal}}
\label{sec:lem:equal}
The proof is complicated slightly because we have not assumed that $f$ is differentiable.
For $\delta > 0$ let
$p_\delta = \cN(\zeros, \delta \id)$ and $f_\delta(x) = \int_{\R^d} f(x + y) p_\delta(y) \d{y}$
be the convolution of $f$ and the Gaussian $p_\delta$. Note that $f_\delta$ is infinitely differentiable 
and inherits convexity and Lipshitzness from $f$ (all immediate from definitions and a good exercise).
Further, $f_\delta$ converges uniformly to $f_0 \triangleq f$ as $\delta \to 0$.
For any $\delta \ge 0$, let 
\begin{align*}
s_\delta(z) = \E\left[\left(1 - \frac{1}{\lambda}\right) f_\delta(X) + \frac{1}{\lambda} f_\delta((1 - \lambda)X + \lambda z)\right]\,,
\end{align*}
which is the surrogate loss function associated with $f_\delta$. 
By a change of variable, $u = x + \frac{\lambda}{1 - \lambda} z$,
\begin{align*}
s_\delta(z) 
&= \frac{1}{\lambda} \int_{\R^d} f_\delta((1 - \lambda) x + \lambda z) p(x) \d{x} 
= \frac{1}{\lambda} \int_{\R^d} f_\delta((1 - \lambda) u) p\left(u - \frac{\lambda}{1-\lambda} z\right) \d{u}\,.
\end{align*}
Therefore, by exchanging derivatives and integrals and reversing the change of measure,
\begin{align*}
\nabla s_\delta(z)
&= \frac{1}{1 - \lambda} \int_{\R^d} f_\delta((1 - \lambda) u) \Sigma^{-1} \left(u - \frac{\lambda}{1-\lambda} z - \mu\right) p\left(u - \frac{\lambda}{1-\lambda} z\right) \d{u} \\
&= \frac{1}{1 - \lambda} \int_{\R^d} f_\delta((1 - \lambda) x + \lambda z) \Sigma^{-1}(x - \mu) p(x) \d{x}\,.
\end{align*}
Using the uniform convergence of $f_\delta$ to $f$ as $\delta \to 0$ yields 
$\lim_{\delta \to 0} \norm{\nabla s_\delta(z) - \nabla s(z)} = 0$ for all $z \in \R^d$. For the Hessian,
\begin{align*}
\nabla^2 s_\delta(z) 
&= \frac{\lambda}{(1 - \lambda)^2} \E\left[f_\delta((1 - \lambda) X + \lambda z) \Sigma^{-1}((X - \mu)(X - \mu)^\top \Sigma^{-1} - \id)\right]\,.
\end{align*}
Hence, $\lim_{\delta \to 0} \norm{\nabla^2 s_\delta(z) - \nabla^2 s(z)} = 0$.
\begin{enumerate}
\item Convexity of $s$ and $s(z) \le f(z)$ for all $z \in \R^d$ follow from the definition of $s$ and the convexity of $f$, as was noted already by \citet{LG21a}, with the intuition given in \Cref{fig:g}.
\item By definition $(1 - \lambda) X + \lambda Z$ has the same law as $X$. Therefore,
\begin{align*}
\E[s(Z)] 
= \E\left[\left(1 - \frac{1}{\lambda}\right) f(X) + \frac{1}{\lambda} f((1 - \lambda)X + \lambda Z)\right] 
&= \E\left[f(X)\right]\,.
\end{align*}
\item By exchanging the integral and derivative, for any $\delta>0$,
\begin{align*}
\E\left[\nabla s_\delta(Z)\right]
= \E\left[\nabla f_\delta((1 - \lambda)X + \lambda Z)\right]
= \E\left[\nabla f_\delta(X)\right]
= \E\left[f_\delta(X) \Sigma^{-1}(X - \mu)\right] \,.
\end{align*}
Taking the limit as $\delta \to 0$ establishes the part.\footnote{Taking the limit with respect to $\delta$ is needed because unlike $f_\delta$ for $\delta>0$, $f$ may not be differentiable.}
\item As above,
\begin{align*}
\E\left[\nabla^2 s_\delta(Z)\right]
&= \lambda \E\left[\nabla^2 f_\delta((1 - \lambda) X + \lambda Z)\right] \\
&= \lambda \E\left[\nabla^2 f_\delta(X)\right] \\
&= \lambda \E\left[f_\delta(X) \Sigma^{-1}((X - \mu)(X - \mu)^\top \Sigma^{-1} - \id)\right] \,.
\end{align*}
Taking the limit again completes the part.
\item This is a consequence of the Lipschitzness of $f$: Let $u \in \R^d$ have $\norm{u} = 1$. Since $f$ is Lipschitz, 
so is $f_\delta$, which means that 
$\norm{\nabla f_\delta(w)} \le 1$ and $|\ip{u,\nabla f_\delta(w)}| \le 1$ for any $w \in \R^d$. Therefore, for any $z \in \R^d$,
\begin{align*}
u^\top \nabla^2 s_\delta(z) u
&= \frac{\lambda}{1 - \lambda} \E\left[\ip{u, \nabla f_\delta((1 - \lambda) X + \lambda z)} \ip{u, \Sigma^{-1}(X - \mu)} \right] \\
& \le \frac{\lambda}{1 - \lambda} \E\left[ \left|\ip{u, \Sigma^{-1}(X - \mu)}\right|\right] \\
&\leq \frac{\lambda}{1 - \lambda} \sqrt{\E\left[\ip{u, \Sigma^{-1}(X - \mu)}^2\right]}  \\
&= \frac{\lambda}{1-\lambda} \sqrt{\norm{u}^2_{\Sigma^{-1}}}   \\ 
&\leq \frac{\lambda}{1 - \lambda} \norm{\Sigma^{-1/2}}\,.
\end{align*}
Therefore $\norm{\nabla^2 s_\delta(z)} \leq \frac{\lambda}{1 - \lambda} \norm{\Sigma^{-1/2}} \le \norm{\Sigma^{-1/2}}$ for all $\delta$ and the result follows
again by taking the limit as $\delta \to 0$.
\item Recall that $\epsilon = \frac{\lambda}{1 - \lambda}(z - w)$ 
and define the event $E = \{x \in \R^d : \ip{x - \mu, \Sigma^{-1} \epsilon} \leq \log(2)\}$.
Then,
\begin{align*}
&\nabla^2 s_\delta(z) 
= \lambda \int_{\R^d} \nabla^2 f_{\delta}((1 - \lambda) x + \lambda z) p(x) \d{x} \\
&= \lambda \int_{\R^d} \nabla^2 f_{\delta}((1 - \lambda) x + \lambda w) p\left(x - \epsilon\right) \d{x} \\
&= \underbracket{\lambda \int_E \nabla^2 f_{\delta}((1 - \lambda)x + \lambda w) \frac{p(x - \epsilon)}{p(x)} p(x) \d{x}}_{\textrm{A}} + \underbracket{\lambda \int_{E^c} \nabla^2 f_{\delta}((1 - \lambda)x + \lambda w) p(x - \epsilon) \d{x}}_{\textrm{B}}\,.
\end{align*}
The first term is upper bounded as
\begin{align*}
\textrm{A} 
&= \lambda \int_{E} \nabla^2 f_{\delta}((1 - \lambda)x + \lambda w) \frac{p(x - \epsilon)}{p(x)} p(x) \d{x} \\
&= \lambda \int_E \nabla^2 f_{\delta}((1 - \lambda)x + \lambda w) \exp\left(-\frac{1}{2} \norm{\epsilon}^2_{\Sigma^{-1}} + \ip{x - \mu, \Sigma^{-1} \epsilon}\right) p(x) \d{x} \\
\tag*{Definition of $E$}
&\leq 2 \lambda \int_E \nabla^2 f_{\delta}((1 - \lambda) x + \lambda w) p(x) \d{x} \\
\tag*{Convexity of $f$}
&\leq 2 \lambda \int_{\R^d} \nabla^2 f_{\delta}((1 - \lambda) x + \lambda w) p(x) \d{x} \\
&= 2 \nabla^2 s_\delta(w) \,.
\end{align*}
To bound $\textrm{B}$, similarly to the calculations in part \eref{(e)}, we have
\begin{align*}
\norm{\textrm{B}}
&= \lambda \sup_{u : \norm{u} \leq 1} \tr\left(uu^\top \int_{E^c} \nabla^2 f_{\delta}((1 - \lambda)x + \lambda w) p(x - \epsilon) \d{x}\right) \\
&= \frac{\lambda}{1 - \lambda} \sup_{u : \norm{u} \leq 1} \int_{E^c} \ip{u, \nabla f_\delta((1 - \lambda)x + \lambda w)} \ip{u, \Sigma^{-1} (x - \mu - \epsilon)} p(x - \epsilon) \d{x} \\
&\leq \frac{\lambda}{1-\lambda} \sup_{u : \norm{u} \leq 1} \int_{E^c} \left|\ip{u, \Sigma^{-1} (x - \mu - \epsilon)}\right| p(x - \epsilon) \d{x} \\
&\leq \frac{\lambda}{1-\lambda} \sup_{u : \norm{u} \leq 1} \sqrt{\int_{E^c} p(x - \epsilon) \d{x} \cdot  \int_{\R^d} \ip{u, \Sigma^{-1} (x - \mu - \epsilon)}^2 p(x - \epsilon) \d{x}} \\
&= \frac{\lambda}{1-\lambda} \sup_{u : \norm{u} \leq 1} \sqrt{\int_{E^c} p(x - \epsilon) \d{x} \cdot \norm{u}^2_{\Sigma^{-1}}} \\ 
&\leq \frac{\lambda}{1-\lambda} \norm{\Sigma^{-1/2}} \sqrt{\int_{E^c} p(x - \epsilon) \d{x}} \\
&= \frac{\lambda}{1-\lambda} \norm{\Sigma^{-1/2}} \sqrt{\bbP\left(\ip{X + \epsilon - \mu, \Sigma^{-1} \epsilon} \geq \log(2)\right)}\,,
\end{align*}
where the second inequality holds by Cauchy-Schwartz, and the non-negativity of the terms in the second integral.
Note that $\ip{X - \mu, \Sigma^{-1} \epsilon}$ has law $\cN(0, \norm{\epsilon}^2_{\Sigma^{-1}})$. Hence, by standard
Gaussian concentration \citep[\S2.2]{BLM13}, if
\begin{align*}
\bbP\left(\ip{X - \mu, \Sigma^{-1} \epsilon} + \norm{\epsilon}^2_{\Sigma^{-1}} \geq \log(2)\right) 
&\leq \exp\left(- \frac{\left(\log(2) - \norm{\epsilon}^2_{\Sigma^{-1}}\right)^2}{2 \norm{\epsilon}^2_{\Sigma^{-1}}}\right) 
\leq \frac{1}{\cnstP}\,,
\end{align*}
where in the final inequality we used the assumption that
$\norm{\epsilon}^2_{\Sigma^{-1}} \leq \frac{\log(2)^2}{2 \log(\cnstP)}$, which also implies $\norm{\epsilon}^2_{\Sigma^{-1}} \le \log(2)$ as $\cnstP\ge 2$, which is necessary for the application of the concentration inequality.
The result follows since $\lambda/(1-\lambda) \leq 1$.
\item No particular properties of $s$ are needed here beyond twice differentiability and that $s$ is Lipschitz, which ensures
that $\lim_{t \to \infty} \int_{\R^d : \norm{z} \geq t} s(z) q(z) \d{z} = 0$.
By definition and integrating by parts,
\begin{align*}
\E[\ip{\nabla s(Z), Z - \mu}]
&= \tr\left(\int_{\R^d} \nabla s(z) (z - \mu)^\top q(z) \d{z}\right) \\
&= -\beta^2 \tr\left(\Sigma \int_{\R^d} \nabla s(z)  \nabla q(z) \d{z} \right) \\
&= \beta^2 \tr\left(\Sigma \int_{\R^d} \nabla^2 s(z) q(z) \d{z} \right) \\
&= \beta^2 \E\left[\tr(\Sigma \nabla^2 s(Z))\right]\,,
\end{align*}
where in the second equality we used the fact that $\nabla q(z) = -\beta^{-2} \Sigma^{-1}(z - \mu) q(z)$ and the cyclic property of the trace.
The third equality follows using integrating by parts. 
\end{enumerate}

%%%%%%%%%%%%%%%%%%%%%%%%%%%%%%%%%%%%%%%%%%%%%%%%%%
% BIAS LEMMA
%%%%%%%%%%%%%%%%%%%%%%%%%%%%%%%%%%%%%%%%%%%%%%%%%%
\section{Proof of Lemma~\ref{lem:bias}}\label{app:lem:bias}
The conceptual part of this proof is straightforward and important for understanding the main ideas.
Sadly there is also a tedious part, which involves handling the truncation used in the gradient and Hessian estimates.

\paragraph{Conceptual part}
Let $I_t = 1 - T_t$.
By definition,
\begin{align*}
\E_{t-1}[\gradt] 
&= \E_{t-1}[D_t \Sigma_t^{-1}(X_t - \mu_t)] \\
&= \E_{t-1}[(Y_t - Y_{t-1}) \Sigma_t^{-1}(X_t - \mu_t)] - \underbracket{\E_{t-1}[(Y_t - Y_{t-1}) I_t \Sigma_t^{-1} (X_t - \mu_t)]}_{\cE_1}
\end{align*}
The second (error) term is intuitively small because $I_t = 0$ with overwhelming probability. Carefully bounding this is the tedious part.
The first term satisfies
\begin{align*}
\E_{t-1}[(Y_t - Y_{t-1}) \Sigma_t^{-1}(X_t - \mu_t)]  
&= \E_{t-1}[Y_t \Sigma_t^{-1}(X_t - \mu_t)] \\
&= \E_{t-1}[f(X_t) \Sigma_t^{-1}(X_t - \mu_t)]  \\
&= \E_{t-1}[\nabla s_t(Z_t)]\,,
\end{align*}
where in the first equality we used that $Y_{t-1}$ and $\Sigma_t$ are $\cF_{t-1}$-measurable and $\E_{t-1}[X_t] = \mu_t$.
In the second equality we substituted the definition of $Y_t = f(X_t) + \epsilon_t$ 
and used the assumption that the noise is conditionally zero mean.
The last follows from Lemma~\ref{lem:equal}\er{c}.
Part \eref{(a)} follows by showing that $|\ip{\mu_t, \cE_1}| \leq \lambda / n$, which we do in the next step.
Moving now to the Hessian, the same reasoning yields
\begin{align*}
\E_{t-1}[\hesst]
&= \lambda \E_{t-1}[(Y_t - Y_{t-1}) \Sigma_t^{-1/2}(W_t  W_t^\top - \id) \Sigma_t^{-1/2}] \\ 
&\qquad - \underbracket{\lambda \E_{t-1}[I_t(Y_t - Y_{t-1}) \Sigma_t^{-1/2}(W_t  W_t^\top - \id) \Sigma_t^{-1/2}]}_{\cE_2}\,. 
\end{align*}
Repeating again the argument above with the first term,
\begin{align*}
\lambda \E_{t-1}[(Y_t - Y_{t-1}) \Sigma_t^{-1/2}(W_t  W_t^\top - \id) \Sigma_t^{-1/2}] 
&= \lambda \E_{t-1}[f(X_t) \Sigma_t^{-1/2}(W_t  W_t^\top - \id) \Sigma_t^{-1/2}] \\ 
&= \E_{t-1}[\nabla^2 s_t(Z_t)]\,,
\end{align*}
where we used Lemma~\ref{lem:equal}\er{d}.
Part \eref{(b)} follow bounding $|\tr(A \cE_2)|$.

\paragraph{Tedious part}
Now we handle the error terms $\cE_1$ and $\cE_2$.
On $\{t \leq \tau\}$, by Definition~\ref{def:E}\er{a}, $\Sigma_t \leq 2\Sigma_1 = \frac{2 \diam^2}{d^2} \id$.
Therefore
\begin{align}
\label{eq:Dmax-Sigma}
\frac{\Dmax}{2} 
\geq \left(1 + 2 \norm{\Sigma_t}^{1/2}\right) \sqrt{\log(2 \cnstP)}\,.
\end{align}
Using Definition~\Cref{def:E}\er{c}, $|\E_{t-1}[Y_t] - Y_{t-1}| \leq \frac{\Dmax}{2}$. 
Therefore, by Lemma~\ref{lem:lip-conc}\er{b}, Fact~\ref{fact:sg}\er{c} and \Cref{lem:sg}\er{a}, 
\begin{align*}
\bbP_{t-1}(I_t)
&\leq\bbP_{t-1}(|Y_t - Y_{t-1}| \geq \Dmax) + \bbP_{t-1}(\norm{W_t} \geq \Wmax) \\
&\leq \bbP_{t-1}\left(|Y_t - \E_{t-1}[Y_t]| \geq \Dmax/2\right) + \bbP_{t-1}(\norm{W_t} \geq \Wmax) \\
&\leq \bbP_{t-1}\left(|Y_t - \E_{t-1}[Y_t]| \geq \left(1 + 2 \norm{\Sigma_t}^{1/2}\right) \sqrt{\log(4 \cnstP)}\right) + \bbP_{t-1}(\norm{W_t} \geq \Wmax)  \\
&\leq \frac{1}{\cnstP}\,.
\end{align*}
We also need a crude bound on the moments of $Y_t - Y_{t-1}$. Again, using Definition~\ref{def:E}\er{c}, \Cref{lem:lip-conc}\er{b}, and
\Cref{eq:Dmax-Sigma} noting that $\cnstP \ge 3$, we obtain
\begin{align*}
\norm{Y_t - Y_{t-1}}_{t-1,\psi_2}
&\leq \norm{Y_t - \E_{t-1}[Y_t]}_{t-1,\psi_2} + \frac{\Dmax}{2\sqrt{\log(2)}} \\ 
&\leq 1 + 2 \norm{\Sigma_t}^{1/2} + \frac{\Dmax}{2\sqrt{\log(2)}}  \\
& \leq \frac{\Dmax}{2\sqrt{\log(2\cnstP)}} + \frac{\Dmax}{2\sqrt{\log(2)}} \\
& \le \Dmax~.
\end{align*}
Therefore, by Lemma~\ref{lem:sg}\er{c},
\begin{align*}
\E[(Y_t - Y_{t-1})^4]^{1/4} \le 6 \Dmax\,.
\end{align*}
Two applications of the Cauchy-Schwarz inequality yield
\begin{align*}
|\ip{\mu_t, \cE_1}| 
&= \left|\E_{t-1}[I_t \ip{\mu_t, (Y_t - Y_{t-1}) \Sigma^{-1}_t(X_t - \mu_t)}]\right| \\
&\leq \bbP_{t-1}(I_t)^{1/4} \E_{t-1}[(Y_t - Y_{t-1})^4]^{1/4} \E[\ip{\mu_t, \Sigma_t^{-1}(X_t - \mu_t)}^2]^{1/2} \\
&\leq \frac{6\Dmax \normt{t}{\mu_t}}{\cnstP^{1/4}} \\
&\leq \frac{6\Dmax \Fmax^{1/2}}{\cnstP^{1/4}} \\
\tag*{By choosing $\cnstP$ large enough}
&\leq \frac{\lambda}{n}\,.
\end{align*}
This completes the proof of part \eref{(a)}. 
Repeating the Cauchy-Schwarz from the last step 
and letting $B = \Sigma_t^{-1/2} A \Sigma_t^{-1/2}$ and using Lemma~\ref{lem:gaussian}\er{c},
\begin{align*}
|\tr(A \cE_2)| 
&=\lambda \left|\E_{t-1}\left[I_t (Y_t - Y_{t-1}) \tr\left(\Sigma_t^{-1/2} A \Sigma_t^{-1/2}(W_t W_t^\top - \id)\right)\right]\right| \\
&\leq \frac{6\lambda \Dmax \sqrt{(d^2+2d-1)\tr(B^2)}}{\cnstP^{1/4}} \\
&\leq \frac{6\lambda \Dmax (d+1)\tr(B)}{\cnstP^{1/4}} \\
&= \frac{6\lambda \Dmax (d+1) \tr(\Sigma_t^{-1} A)}{\cnstP^{1/4}} \\
&\leq \tr(\Sigma_t^{-1} A) \min\left(\frac{\lambda}{2n \Fmax},\, \frac{\lambda}{dn},\, \frac{1}{n \Sigma_{\max}^{-1}}\right)\,,
\end{align*}
where in the final inequality we chose $\cnstP$ large enough.

%%%%%%%%%%%%%%%%%%%%%%%%%%%%%%%%%%%%%%%%%%%%%%%%%%%%%%%%%%%%%
% PROOF OF SECOND HESSIAN LEMMA 
%%%%%%%%%%%%%%%%%%%%%%%%%%%%%%%%%%%%%%%%%%%%%%%%%%%%%%%%%%%%%
\section{Proof of Lemma~\ref{lem:hess2}}\label{app:lem:hess2}
Since $\xi_z \in [\zeros, z] = \{\alpha z : \alpha \in [0,1]\}$, it follows from convexity that 
\begin{align*}
\norm{\xi_z - \mu_t}_{\Sigma_t^{-1}} \leq \max\left(\norm{z - \mu_t}_{\Sigma_t^{-1}},\, \normt{t}{\mu_t}\right) \,.
\end{align*}
By assumption $t \le \tau$, which implies that 
$\frac{1}{2}\normt{t}{\mu_t}^2 \leq \Fmax$. Then, using the definition of $\lambda$, 
\begin{align*}
\frac{\lambda}{1 - \lambda} \normt{t}{\mu_t}
\leq \frac{\lambda}{1 - \lambda} \sqrt{2 \Fmax}
\leq \frac{\log(2)}{\sqrt{2 \log(\cnstP)}}\,.
\end{align*}
Recall that $Z_t$ has law $\cN(\mu_t, \beta^2 \Sigma_t)$ under $\bbP_{t-1}$. By Lemma~\ref{lem:gaussian}\er{de},
\begin{align}
\E_{t-1}[\norm{Z_t}^2] = \norm{\mu_t}^2 + \beta^2 \tr(\Sigma_t)  \qquad \text{ and } \qquad
\E_{t-1}[\norm{Z_t}^4] \leq 3 \left(\E_{t-1}[\norm{Z_t}^2]\right)^2 \,.
\label{eq:Zt}
\end{align}
By Fact~\ref{fact:sg}\er{c} and Lemma~\ref{lem:sg}\er{a}, 
with probability at least $1-1/\cnstP$,
\begin{align*}
\frac{\lambda}{1-\lambda} \bnorm{Z_t - \mu_t}_{\Sigma_t^{-1}} 
&\leq \frac{2\lambda \beta}{1 - \lambda} \sqrt{d \log(2\cnstP)} \\ 
\tag*{Since $\beta^2 = (2-\lambda)/\lambda$}
&= \frac{2}{1 - \lambda} \sqrt{\lambda (2 - \lambda) d \log(2\cnstP)} \\
&\leq \frac{\log(2)}{\sqrt{2 \log(2\cnstP)}} \,,
\end{align*}
where the second inequality follows by choosing $\cnst$ in the definition of $\lambda$ small enough.
This shows that
\begin{align}
\label{eq:lemhess2event}
\bbP_{t-1}(Z_t \not\in E) \leq 1/\cnstP
\end{align}
for
\begin{align*}
E = \left\{z \in \R^d: \frac{\lambda}{1 - \lambda} \norm{z - \mu_t}_{\Sigma_t^{-1}} \leq \frac{\log(2)}{\sqrt{2 \log(\cnstP)}}\right\} \,.
\end{align*}
By Lemma~\ref{lem:equal}\er{f}, for $z \in E$ we have
\begin{align*}
\nabla^2 s_t(\xi_z) \geq \frac{1}{2} \nabla^2 s_t(\mu_t) - \frac{\norm{\Sigma_t^{-1/2}}}{2\cnstP} \id \,.
\end{align*}
Therefore,
\begin{align}
&\E_{t-1}\left[\norm{Z_t}^2_{\nabla^2 s_t(\xi_{Z_t})}\right] 
\geq \frac{1}{2} \E_{t-1}\left[\sind(Z_t \in E)\norm{Z_t}^2_{\nabla^2 s_t(\mu_t)}\right] - \frac{\norm{\Sigma_t^{-1/2}}}{2\cnstP} \E[\norm{Z_t}^2] \nonumber \\
&\qquad= \frac{1}{2} \E_{t-1}\left[\norm{Z_t}^2_{\nabla^2 s_t(\mu_t)}\right] 
  - \frac{1}{2} \E_{t-1}\left[\sind(Z_t \not\in E) \norm{Z_t}^2_{\nabla^2 s_t(\mu_t)}\right] - \frac{\norm{\Sigma_t^{-1/2}}}{2\cnstP} \E[\norm{Z_t}^2] \,.
  \label{eq:hess2-1}
\end{align}
The last term in \Cref{eq:hess2-1} is bounded using \Cref{eq:Zt}:
\begin{align}
\frac{\E_{t-1}[\norm{Z_t}^2] \norm{\Sigma_t^{-1/2}}}{2 \cnstP}
&\leq \frac{\E_{t-1}[\norm{Z_t}^2] \sqrt{\Sigma_{\max}^{-1}}}{2 \cnstP} \nonumber \\
&=\frac{\left(\norm{\mu_t}^2 + \beta^2 \tr(\Sigma_t) \right) \sqrt{\Sigma_{\max}^{-1}}}{2 \cnstP} \nonumber \\
&\leq \frac{\left(2\norm{\Sigma_1} \Fmax + 2\beta^2 \tr(\Sigma_1)\right) \sqrt{\Sigma_{\max}^{-1}}}{2 \cnstP} \nonumber \\
&\leq \frac{1}{2n}\,, \label{eq:ZSigma}
\end{align}
where the last inequality follows by choosing $\cnstP$ large enough and 
in the first and second inequalities we used the facts that on $\{t \leq \tau\}$,
\begin{align*}
\norm{\Sigma_t^{-1/2}} &\leq \sqrt{\tr(\Sigma_t^{-1})} \leq \sqrt{\Sigma_{\max}^{-1}} \qquad\text{and}\qquad
\norm{\mu_t}^2 \leq \norm{\Sigma_t} \normt{t}{\mu_t}^2 \leq 2\norm{\Sigma_1} \Fmax \,.
\end{align*}
For the second to last term in \Cref{eq:hess2-1}, \Cref{lem:equal}\er{e}, \Cref{eq:Zt} and \Cref{eq:lemhess2event} imply
\begin{align*}
\frac{1}{2} \E_{t-1}\left[\sind(Z_t \not\in E) \norm{Z_t}^2_{\nabla^2 s_t(\mu_t)}\right] 
&\leq \frac{1}{2}\norm{\Sigma_t^{-1/2}} \E_{t-1}\left[\sind(Z_t \not\in E) \norm{Z_t}^2\right] \\ 
&\leq \frac{1}{2}\norm{\Sigma_t^{-1/2}} \sqrt{\bbP_{t-1}(Z_t \not\in E) \E_{t-1}\left[\norm{Z_t}^4\right]} \\
&\leq \E_{t-1}\left[\norm{Z_t}^2\right] \norm{\Sigma_t^{-1/2}} \sqrt{\frac{3}{4\cnstP}} \\
&\leq \frac{1}{2n}\,,
\end{align*}
where the last inequality follows the same way as \Cref{eq:ZSigma}, again making sure that $\cnstP$ is chosen large enough. Therefore,
\begin{align*}
\E_{t-1}\left[\norm{Z_t}^2_{\nabla^2 s_t(\xi_{Z_t})}\right] 
\geq \frac{1}{2} \E_{t-1}\left[\norm{Z_t}^2_{\nabla^2 s_t(\mu_t)}\right] - \frac{1}{n}\,.
\end{align*}
Part \eref{(b)} follows along the same lines. Using \Cref{lem:equal}\er{ef},	
\begin{align*}
\frac{1}{2} \E_{t-1}\left[\norm{\mu_t}^2_{\nabla^2 s_t(Z_t)}\right]
&= \frac{1}{2} \E_{t-1}\left[\sind(Z_t\in E) \norm{\mu_t}^2_{\nabla^2 s_t(Z_t)}\right] 
    + \frac{1}{2} \E_{t-1}\left[\sind(Z_t\not\in E) \norm{\mu_t}^2_{\nabla^2 s_t(Z_t)}\right] \\ 
&\leq \E_{t-1}\left[\norm{\mu_t}^2_{\nabla^2 s_t(\mu_t)}\right] 
+ \frac{\norm{\Sigma_t^{-1/2}}}{2\cnstP} \E_{t-1}\left[ \norm{\mu_t}^2 \right] \\
 & \qquad   + \frac{\norm{\Sigma_t^{-1/2}}}{2} \E_{t-1}\left[\sind(Z_t\not\in E) \norm{\mu_t}^2\right] \\ 
&\leq \E_{t-1}\left[\norm{\mu_t}^2_{\nabla^2 s_t(\mu_t)}\right] 
    + \frac{\norm{\Sigma_t^{-1/2}}\norm{\Sigma_t} \Fmax}{\cnstP} \\
&\leq \E_{t-1}\left[\norm{\mu_t}^2_{\nabla^2 s_t(\mu_t)}\right]  
    + \frac{1}{n}\,,
\end{align*}
where in the the second to last inequality we used the independence of $Z_t$ and $\mu_t$ under $\bbP_{t-1}$ and \Cref{eq:lemhess2event}, while
the final inequality we again used that on $\{t \le \tau\}$ both $\norm{\Sigma_t^{-1/2}}$ and $\norm{\Sigma_t}$ are
bounded by polynomials in $d$, $n$ and $\diam$.

%%%%%%%%%%%%%%%%%%%%%%%%%%%%%%%%%%%%%%%%%%%%%%%%%%%%%%%%%%%%%%%%%%%
% STABILITY 
%%%%%%%%%%%%%%%%%%%%%%%%%%%%%%%%%%%%%%%%%%%%%%%%%%%%%%%%%%%%%%%%%%%
\section{Proof of \Cref{lem:stab}}\label{app:lem:stab}
By definition $E_t = E_t^{(a)} \cap E_t^{(b)} \cap E_t^{(c)}$, where
\begin{align*}
E_t^{(a)} &= \{ \Sigma_{t+1} \leq 2\Sigma_1\}, & 
E_t^{(b)} &= \{ \tr(\Sigma_{t+1}^{-1}) \leq \Sigma_{\max}^{-1} \}, &
E_t^{(c)} &= \{ \left|\E_t[Y_{t+1}] - Y_t\right| \leq \Dmax / 2\}\,.
\end{align*}
The plan is to show that all of these events occur with high probability; then a naive application of the union bound finishes the proof.

%%%%%%%%%%%%%%%%%%%%%%%%%%%%%%%%%%%%%%%%%%%%%%%%%%%%%%%%%%%%%%%%%%%%
\paragraph{Step 1: Stability}
%%%%%%%%%%%%%%%%%%%%%%%%%%%%%%%%%%%%%%%%%%%%%%%%%%%%%%%%%%%%%%%%%%%%
We start by showing that the mean and covariance change slowly, which is a consequence of the truncation in the algorithm
and the choices of parameters. On the event $\{t \leq \tau\}$
\begin{align}
\label{eq:muchange}
\norm{\mu_{t+1} - \mu_t}
&= \eta \norm{\Sigma_t \gradt} 
= \eta \norm{D_t (X_t - \mu_t)} 
\leq \eta \Dmax \Wmax \norm{\Sigma_t}^{1/2}  
\leq \frac{\diam}{d}\,.
\end{align}
Moving now to the covariance, we have
\begin{align}
\label{eq:sigma-change}
\Sigma_{t+1}^{-1} = \Sigma_t^{-1} + \frac{\eta}{4} \hesst
= \Sigma_t^{-1/2}\left(\id + \frac{\eta}{4}\Sigma_t^{1/2} \hesst \Sigma_t^{1/2}\right) \Sigma_t^{-1/2}\,.
\end{align}
By the definition of $\hesst$,
\begin{align}
\label{eq:sigmaHbound}
\bnorm{\frac{\eta}{4} \Sigma_t^{1/2} \hesst \Sigma_t^{1/2}}
&= \bnorm{\frac{\eta \lambda D_t}{4} (W_t W_t^\top - \id)}  \\
&\leq \frac{\eta \lambda \Dmax (1 + \Wmax^2)}{4} \\
\tag*{Definitions of $\eta$ and $\lambda$}
&\leq \frac{1}{2d}\,,
\end{align}
where the last inequality follows by choosing $\cnst$, $\Cnst$ and $\cnstP$ sufficiently small, large and large, respectively.

%%%%%%%%%%%%%%%%%%%%%%%%%%%%%%%%%%%%%%%%%%%%%%%%%%%%%%%%%%%%%%%%%%%%
\paragraph{Step 2: Baseline quality}
%%%%%%%%%%%%%%%%%%%%%%%%%%%%%%%%%%%%%%%%%%%%%%%%%%%%%%%%%%%%%%%%%%%%
The plan in this step is to show that on $\{t \leq \tau\}$, 
\begin{align}
\bbP_{t-1}\left(\left|\E_t[Y_{t+1}] - Y_t\right| \leq \frac{\Dmax}{2}\right) \leq \frac{1}{n^2} \,.
\label{eq:baseline}
\end{align}
and then use a union bound to establish that $\bbP(\cap_{t=1}^\tau E_t^{(c)}) \geq 1-1/n$.
There are two parts to establishing \Cref{eq:baseline}: 
\begin{enumeratenum}
\item Showing that $Y_t$ is close to $\E_{t-1}[Y_t]$, which follows from concentration for Lipschitz
functions and the definition of the noise model. 
\item Showing that $\E_{t-1}[Y_t]$ is close to $\E_t[Y_{t+1}]$, which is a consequence of the stability of the algorithm that
was shown in the previous step.
\end{enumeratenum}
We start with the second. By Lemma~\ref{lem:norm}\er{ab} and using that 
$\frac{1}{4}\eta \norm{\Sigma_t^{1/2} \hesst \Sigma_t^{1/2}} \leq \frac{1}{2d}$,
\begin{align*}
\frac{2d-1}{2d}\Sigma_t^{-1} \leq \Sigma_{t+1}^{-1} \leq \frac{2d+1}{2d} \Sigma_t^{-1} 
\quad\text{ and so }\quad
\frac{2d}{2d+1} \Sigma_t \leq \Sigma_{t+1} \leq \frac{2d}{2d-1} \Sigma_t\,.
\end{align*}
Because $f$ is Lipschitz, if the mean and covariance matrices change slowly from one round to the next, then
the mean loss should also change slowly.
This phenomenon is captured by Lemma~\ref{lem:normal-diff}, which yields
\begin{align}
\left|\E_t[Y_{t+1}] - \E_{t-1}[Y_t]\right| 
&\leq \sqrt{\norm{\mu_{t+1} - \mu_t}^2 + \tr\left(\Sigma_t + \Sigma_{t+1} - 2 \left(\Sigma_t^{1/2} \Sigma_{t+1} \Sigma_t^{1/2}\right)^{1/2}\right)} \\
&\leq \sqrt{\frac{\diam^2}{d^2} + \tr\left((1 - \sqrt{1-1/(2d)}) \Sigma_t + (1 - \sqrt{1-1/(2d)}) \Sigma_{t+1}\right)} \\
&\leq \sqrt{\frac{\diam^2}{d^2} + \frac{1}{2d} \tr\left(\Sigma_t + \Sigma_{t+1}\right)} \\
&\leq \frac{3 \diam}{d}\,,
\end{align}
where the second inequality follows by \Cref{eq:muchange}-\Cref{eq:sigmaHbound} and standard monotonicity properties of the trace.

On $\{t \leq \tau\}$, $E_{t-1}^{(b)}$ holds and so $\norm{\Sigma_t} \leq \frac{2\diam^2}{d^2}$ and
by \Cref{lem:lip-conc}\er{b} and \Cref{lem:sg}\er{a}, 
\begin{align*}
\bbP_{t-1}\left(|Y_t - \E_{t-1}[Y_t]| \geq \left(1 + \frac{2\diam}{d}\right) \sqrt{2\log(2n^2)}\right) \leq \frac{1}{n^2} \,.
\label{eq:stab:Y}
\end{align*}
Hence, combining the above two displayed inequalities, it follows that
with $\bbP_{t-1}$-probability at least $1 - 1/n^2$,
\begin{align*}
\left|\E_t[Y_{t+1}] - Y_t\right|
&\leq \left|\E_t[Y_{t+1}] - \E_{t-1}[Y_t]\right| + \left|\E_{t-1}[Y_t] - Y_t\right|
\leq \left(1 + \frac{5\diam}{d}\right) \sqrt{2 \log(2n^2)} 
\leq \frac{\Dmax}{2}\,,
\end{align*}
where the last inequality follows by choosing the constants in the definition of $\Dmax$ large enough.
By a union bound it now follows that $\bbP(\cap_{t=1}^\tau E_t^{(c)}) \geq 1 - 1/n$.

%%%%%%%%%%%%%%%%%%%%%%%%%%%%%%%%%%%%%%%%%%%%%%%%%%%%%%%%%%%%%%%%%%%%
\paragraph{Step 3: Lower bound on covariance}
%%%%%%%%%%%%%%%%%%%%%%%%%%%%%%%%%%%%%%%%%%%%%%%%%%%%%%%%%%%%%%%%%%%%
In this step we show that with probability at least $1 - 1/n$ 
\begin{align*}
\tr\left(\Sigma_{\tau+1}^{-1}\right) < \Sigma_{\max}^{-1}\,,
\end{align*}
which implies that 
$\bbP(\cap_{t=1}^\tau E_t^{\texttt{(b)}}) = \bbP\left(\tr(\Sigma_{\tau+1}^{-1}\right) \leq \Sigma_{\max}^{-1}) \geq 1 - 1/n$.
Since $\Sigma_t^{-1}$ is positive definite for all $t$, by Markov's inequality,
\begin{align*}
\bbP\left(\tr(\Sigma_{\tau+1}^{-1}) \geq \Sigma_{\max}^{-1}\right) 
&\leq \frac{1}{\Sigma_{\max}^{-1}} \E\left[\tr(\Sigma_{\tau+1}^{-1})\right] \\
&= \frac{1}{\Sigma_{\max}^{-1}} \left(\tr(\Sigma_1^{-1}) + \frac{\eta}{4} \E\left[\sum_{t=1}^n \tr(H_t) \sind_{\tau \geq t}\right]\right) \\
\tag*{\Cref{lem:bias}\er{b}}
&\leq \frac{1}{\Sigma_{\max}^{-1}} \left(\frac{d^3}{\diam^2} + \frac{\eta}{4} \E\left[\sum_{t=1}^n \tr(\nabla^2 s_t(Z_t)) \sind_{\tau \geq t}\right] + \frac{\eta}{4}\right) \\
\tag*{\Cref{lem:equal}\er{e}}
&\leq \frac{1}{\Sigma_{\max}^{-1}} \left(\frac{d^3}{\diam^2} + \frac{\eta n d \sqrt{\Sigma_{\max}^{-1}}}{4} + \frac{\eta}{4}\right) \\
&\leq \frac{1}{\sqrt{\Sigma_{\max}^{-1}}} \left(\frac{d^2}{\diam^2} + \frac{n d}{4} + 1\right) \\
&\leq \frac{1}{n} \,,
\end{align*}
where the second to last inequality follows by naive simplification and the last using the definition of $\Sigma_{\max}^{-1}$ in
\Cref{def:E}\er{b}.

%%%%%%%%%%%%%%%%%%%%%%%%%%%%%%%%%%%%%%%%%%%%%%%%%%%%%%%%%%%%%%%%%%%%
\paragraph{Step 4: Upper bound on covariance}
%%%%%%%%%%%%%%%%%%%%%%%%%%%%%%%%%%%%%%%%%%%%%%%%%%%%%%%%%%%%%%%%%%%%
It remains to show that $\Sigma_{\tau+1}\le 2 \Sigma_1$ with high probability, which
we do by showing that $\Sigma_{t+1}^{-1}$ is unlikely to decrease too much.
The sphere embedded in $\R^d$ is denoted by
$\sphere = \{x \in \R^d : \norm{x} = 1\}$ and let $x \in \sphere$ be an arbitrary unit vector. 
The update of $\hesst$ guarantees that, for any $s \le \tau$,
\begin{align*}
\normt{s+1}{x}^2
=  \normt{1}{x}^2 + \sum_{t=1}^s \frac{1}{4} \eta \norm{x}^2_{\hesst}\,,
\end{align*}
where you should note that $\normt{1}{x}^2 = \norm{x}^2_{\Sigma_1^{-1}}$ is not the $1$-norm.
Therefore,
\begin{align*}
\log\left(\frac{ \normt{\tau+1}{x}^2}{ \normt{1}{x}^2}\right)
&= \sum_{t=1}^\tau \log\left(1 + \frac{\frac{1}{4} \eta \norm{x}^2_{\hesst}}{ \normt{t}{x}^2}\right)\,. \\
\end{align*}
Next we lower bound the terms in the sum. Fix some $t \le \tau$. By Fact~\ref{fact:sg}\er{b} and simplifying, 
\begin{align*}
\bnorm{\frac{\frac{1}{4} \eta \norm{x}^2_{\hesst}}{ \normt{t}{x}^2}}_{t-1,\psi_1}
&= 
\frac{\eta \lambda }{4 \normt{t}{x}^2} \bnorm{D_t \norm{\Sigma_t^{-1/2} x}^2_{W_t W_t^\top - \id}}_{t-1,\psi_1} 
\leq \frac{5}{4}\eta \lambda \Dmax\,.
\end{align*}
Since $\E_{t-1}[\hesst]$ is very close to the Hessian of a convex function, we should expect that $\E_{t-1}[\norm{x}^2_{\hesst}]$ is nearly positive.
Indeed, by Lemma~\ref{lem:bias}\er{b},
\begin{align*}
\E_{t-1}\left[\frac{\frac{1}{4} \eta \norm{x}^2_{\hesst}}{ \normt{t}{x}^2}\right]
&\geq \frac{\E_{t-1}\left[\norm{x}^2_{\nabla^2 s(Z_t)}\right]}{\normt{t}{x}^2} - \frac{\lambda}{dn}
\geq -\frac{1}{10n}\,.
\end{align*}
Let $\cC \subset \sphere$ be a finite cover of the sphere such that for all $y \in \sphere$ there exists an $x \in \cC$ 
with $\norm{x - y} \leq 1/\cnstP$. By Corollary 4.2.13 of \citet{Ver18}, the cover can be chosen so that 
\begin{align*}
|\cC| \leq \left(2\cnstP + 1\right)^d\,.
\end{align*}
Next, by Bernstein's inequality (Lemma~\ref{lem:bernstein}), 
with probability at least $1 - 1/n$, for all $x \in \cC$,
\begin{align*}
\sum_{t=1}^\tau \frac{\frac{1}{4} \eta \norm{x}^2_{\hesst}}{ \normt{t}{x}^2} \geq -83 \eta \lambda \Dmax \sqrt{n \log(n |\cC|)} - \frac{1}{10}\,.
\end{align*}
Furthermore, by Lemma~\ref{lem:sg}\er{b}, with probability at least $1 - 1/n$, for all $x \in \cC$ and $t \leq \tau$, 
\begin{align*}
\left|\frac{\frac{1}{4} \eta \norm{x}^2_{\hesst}}{ \normt{t}{x}^2}\right| \leq \eta \lambda \Dmax \log(2n^2 |\cC|) \leq 1\,. 
\end{align*}
Therefore, using $\log(1+x) \ge x-x^2$ for $x \ge -1$, we obtain
\begin{align*}
\log\left(\frac{ \normt{\tau+1}{x}^2}{ \normt{1}{x}^2}\right)
&= \sum_{t=1}^\tau \log\left(1 + \frac{\frac{1}{4} \eta \norm{x}^2_{\hesst}}{ \normt{t}{x}^2}\right) \\
&\geq \sum_{t=1}^\tau \left[\frac{\frac{1}{4} \eta \norm{x}^2_{\hesst}}{ \normt{t}{x}^2} - \left(\frac{\frac{1}{4} \eta \norm{x}^2_{\hesst}}{ \normt{t}{x}^2}\right)^2 \right]\\
&\geq -83 \eta \lambda \Dmax \sqrt{n \log(n |\cC|)} - n \eta^2 \lambda^2 \Dmax^2 \log^2(2n^2 |\cC|) - \frac{1}{10} \\
&\geq -\frac{1}{20} > -\log(4/3)\,.
\end{align*}
Combining the above calculations with the analysis in the previous step and a union bound shows that
with probability at least $1 - 3/n$ it holds that $\norm{\Sigma_{\tau+1}^{-1}} \leq \Sigma_{\max}^{-1}$ and for all $x \in \cC$,
\begin{align*}
 \normt{\tau+1}{x}^2 \geq \frac{3  \normt{1}{x}^2}{4} = \frac{3}{4}\norm{x}^2_{\Sigma_1^{-1}} = \frac{3d^2}{4\diam^2}\,. 
\end{align*}
On this event,
\begin{align*}
\min_{y \in \sphere} \normt{\tau+1}{y}
&= \min_{y \in \sphere} \max_{x \in \cC} \normt{\tau+1}{y - x + x} \\
&\geq \min_{y \in \sphere} \max_{x \in \cC} \left(\normt{\tau+1}{x} - \normt{\tau+1}{x - y}\right) \\
&\geq \min_{y \in \sphere} \left(\min_{x \in \cC} \normt{\tau+1}{x} - \min_{x \in \cC} \norm{\Sigma_{\tau+1}^{-1}} \norm{x - y} \right) \\
&\geq \min_{x \in \cC} \normt{\tau+1}{x} - \frac{\norm{\Sigma_{\tau+1}^{-1}}}{\cnstP} \\
&\geq \sqrt{\frac{3d^2}{4\diam^2}} - \frac{\Sigma_{\max}^{-1}}{\cnstP} \\
&\geq \sqrt{\frac{d^2}{2\diam^2}}\,,
\end{align*}
where the last inequality holds for $\cnstP$ large enough.
Therefore with probability at least $1 - 3/n$, $\Sigma_{\tau+1}^{-1} \geq \Sigma_1^{-1} / 2$,
which implies that $\Sigma_{\tau+1} \leq 2 \Sigma_1$ and so 
\begin{align*}
\bbP\left(\bigcap_{t=1}^\tau \left(E_t^{(a)} \cap E_t^{(b)}\right)\right)
= \bbP\left(\Sigma_{\tau+1} \leq 2 \Sigma_1 \text{ and } \tr(\Sigma_{\tau+1}^{-1}) \leq \Sigma_{\max}^{-1}\right) 
\geq 1 - 3/n\,.
\end{align*}

%%%%%%%%%%%%%%%%%%%%%%%%%%%%%%%%%%%%%%%%%%%%%%%%%%
% TECHNICAL LEMMAS
%%%%%%%%%%%%%%%%%%%%%%%%%%%%%%%%%%%%%%%%%%%%%%%%%%
\section{Technical lemmas}

\begin{lemma}\label{lem:norm}
Suppose that $A$, $B$ and $C$ are square matrices with $A$ and $B$ positive definite, $C$ symmetric  
and $\norm{C} \leq \epsilon \leq 1$. Then
\begin{enumerate}
\item $A^{1/2} (\id + C) A^{1/2} \geq (1 - \epsilon) A$ and $A^{1/2} (\id + C) A^{1/2} \leq (1 + \epsilon) A$. 
\item If $A \leq B$, then $B^{-1} \leq A^{-1}$.
\end{enumerate}
\end{lemma}

%%%%%%%%%%%%%%%%%%%%%%%%%%%%%%%%%%%%%%%%%%%%%%%%%%%%%%
% CONCENTRATION
%%%%%%%%%%%%%%%%%%%%%%%%%%%%%%%%%%%%%%%%%%%%%%%%%%%%%%
\section{Concentration bounds}\label{app:conc}

None of the results in this section are novel in any way. In some cases we needed to include explicit constants where published
results simplify with unspecified universal constants. We are expedient in our calculation of these constants.
In case you wanted a truly refined analysis, then the Orlicz-norm style analysis should be replaced with the kind of analysis
that relies on moment-generating functions.

\begin{fact}\label{fact:sg}
Let $W \stackrel{d}= \cN(\zeros, \id)$. Then
\begin{enumerate}
\item $\norm{\ip{x, W}}_{\psi_2} = 2\sqrt{2/3} \norm{x} \leq 2\norm{x}$ for any $x \in \R^d$;
\item $\norm{\tr(A WW^\top)}_{\psi_1} \leq 3 \tr(A)$ for any positive semidefinite $A \in \R^{d \times d}$;
\item $\bignorm{\norm{W}}_{\psi_2}^2 =  \bignorm{\norm{W}^2}_{\psi_1} \leq 8d/3$;
\item $\bignorm{\norm{WW^\top - \id}}_{\psi_1} \leq 5d$.
\end{enumerate}
\end{fact}

\begin{proof}
All results follow from explicit calculation using the Gaussian density:
\begin{enumerate}
\item Since $\norm{\ip{x, W}}_{\psi_2} = \norm{x} \bignorm{\ip{x, W}/\norm{x}}_{\psi_2}$, we may assume without loss of generality that $\norm{x} = 1$
so that $\ip{x, W} \stackrel{d}= \cN(0, 1)$. Then
\begin{align*}
\E[\exp(\ip{x, W}^2/t^2)] = \frac{1}{\sqrt{2\pi}} \int_{-\infty}^\infty \exp(-z^2/2+z^2/t^2) \d{z} = 1/\sqrt{1-2/t^2} \,.
\end{align*}
The right-hand side is less than $2$ for $t \geq 2\sqrt{2/3} \approx 1.63$.
\item Let $A = U^{-1} \Lambda U$ for orthonormal $U$ and $\Lambda$ a diagonal matrix with eigenvalues $\lambda_1,\ldots,\lambda_d$.
By rotational invariance
$UW$ and $W$ have the same distribution and therefore so do $\tr(A WW^\top)$ and $\norm{W}^2_{\Lambda}$.
Therefore $\norm{\tr(AWW^\top)}_{\psi_1} = \norm{\sum_{m=1}^d \lambda_m W_m^2}_{\psi_1} \leq \sum_{m=1}^d \lambda_m \norm{W_m^2}_{\psi_1}$.
The result follows since, by part \eref{(a)},  $\norm{W_m^2}_{\psi_1} = \norm{W_m}_{\psi_2}^2 = 8/3 \leq 3$.
\item Using \eref{(b)}, $\bignorm{\norm{W}}_{\psi_2}^2 = \bignorm{\norm{W}^2}_{\psi_1} = \norm{\tr(\id WW^\top)}_{\psi_1} \leq 8d/3$.
\item Using \eref{(b)}, $\bignorm{\norm{WW^\top - \id}}_{\psi_1} \leq \bignorm{1 + \norm{WW^\top}}_{\psi_1} = \log(2) + \norm{\tr(WW^\top)}_{\psi_1} \leq 5 d$.
\end{enumerate}
\end{proof}

\begin{lemma}[Proposition 2.5.2, Proposition 2.7.1, \citealt{Ver18}]\label{lem:sg}
Let $X$ be a real random variable. Then for all $t \geq 0$,
\begin{enumerate}
\item $\bbP(|X| \geq t) \leq 2 \exp\left(-t^2/\norm{X}_{\psi_2}^2\right)$; 
\item $\bbP(|X| \geq t) \leq 2 \exp\left(-t/\norm{X}_{\psi_1}\right)$; 
\item $\E[|X|^p]^{1/p} \leq 3 \sqrt{p} \norm{X}_{\psi_2}$ for all $p \geq 1$; 
\end{enumerate}
\end{lemma}

\begin{lemma}\label{lem:bernstein}
Let $(\Omega, \cF, \bbP)$ be a probability space and
$X_1,\ldots,X_n$ be a sequence of random variables adapted to a filtration $(\cF_t)_{t=1}^n$ and let $\norm{X_t}_{t-1,\psi_1}$ be
the $\norm{\cdot}_{\psi_1}$ norm of $X_t$ with respect to $\bbP(\cdot | \cF_{t-1})$. Suppose that $\tau$ is a stopping time with respect
to $(\cF_t)_{t=1}^n$ and $\norm{X_t}_{t-1,\psi_1} \leq \alpha$ almost surely on $\{t - 1 < \tau\}$. Then for any $\delta \in (0,1)$,
\begin{align*}
\bbP\left(\sum_{t=1}^\tau X_t - \E[X_t|\cF_{t-1}] \geq 66\max \left[\alpha \sqrt{n \log(1/\delta)}\,,\, \alpha \log(1/\delta)\right]\right) \leq \delta\,.
\end{align*}
\end{lemma}

\begin{proof}
Repeat the proof of the standard Bernstein inequality 
to the sequence $(Y_t)_{t=1}^n$ with $Y_t = X_t \sind(t \leq \tau)$ \citep[Theorem 2.8.1]{Ver18}.
\end{proof}

\begin{lemma} \label{lem:normal-diff}
Let $X\stackrel{d}= \cN(\mu_1, \Sigma_1)$ and $Y \stackrel{d}= \cN(\mu_2, \Sigma_2)$ and $f : \R^d \to \R$ be
Lipschitz. Then 
\begin{align*}
|\E[f(X)] - \E[f(Y)]| 
&\leq \sqrt{\norm{\mu_1 - \mu_2}^2 + \tr\left(\Sigma_1 + \Sigma_2 - 2\left(\Sigma_1^{1/2} \Sigma_2 \Sigma_1^{1/2}\right)^{1/2}\right)}\,. 
\end{align*}
\end{lemma}

\begin{proof}
Let $W_k(p, q)$ be the $k$-Wasserstein distance between probability measures $p$ and $q$ for $k \in \{1,2\}$.
Since $f$ is Lipschitz, $|\E[f(X)] - \E[f(Y)]| \leq W_1(\cN(\mu_1, \Sigma_1), \cN(\mu_2, \Sigma_2))$ and by convexity $W_1 \leq \sqrt{W_2}$.
Therefore using the closed form of $W_2$ between Gaussians by \cite{dowson1982frechet} yields
\begin{align*}
\left|\E[f(X)] - \E[f(Y)]\right|
&\leq \sqrt{W_2(\cN(\mu_1, \Sigma_1), \cN(\mu_2, \Sigma_2))} \\
&= \sqrt{\norm{\mu_1 - \mu_2}^2 + \tr\left(\Sigma_1 + \Sigma_2 - 2\left(\Sigma_1^{1/2} \Sigma_2 \Sigma_1^{1/2}\right)^{1/2}\right)}\,. 
\end{align*}
\end{proof}

\begin{lemma}\label{lem:lip-conc}
Suppose that $X$ has law $\cN(\mu, \Sigma)$ and $\norm{\epsilon}_{\bbP(\cdot|X),\psi_2} \leq 1$.
Then with $Y = f(X) + \epsilon$ and any $\delta \in (0,1)$,
\begin{enumerate}
\item $\norm{f(X) - \E[f(X)]}_{\psi_2} \leq 2 \norm{\Sigma}^{1/2}$; and
\item $\norm{Y - \E[Y]}_{\psi_2} \leq 1 + 2\norm{\Sigma}^{1/2}$.
\end{enumerate}
\end{lemma}

\begin{proof}
Let $W$ have law $\cN(\zeros, \id)$ and $h(w) = f(\Sigma^{1/2} w + \mu)$. 
Since $f$ is Lipschitz, $h(u) - h(v) \leq \norm{\Sigma^{1/2}(u - v)} \leq \norm{\Sigma}^{1/2} \norm{u - v}$. Since
this holds for all $u$ and $v$, $h$ is $\norm{\Sigma}^{1/2}$-Lipschitz with respect to the euclidean norm.
By Theorem 5.6 of \citet{BLM13},
\begin{align*}
\bbP\left(|f(X) - \E[f(X)]| \geq t\right) \leq 2 \exp\left(-\frac{t^2}{2 \norm{\Sigma}}\right)\,.
\end{align*}
Therefore, by Proposition~2.5.2 of \citet{Ver18}, $\norm{f(X) - \E[f(X)]}_{\psi_2} \leq 2 \norm{\Sigma}^{1/2}$, which 
establishes \eref{(a)}. By the triangle inequality,
\begin{align*}
\norm{Y - \E[Y]}_{\psi_2}
&= \norm{f(X) + \epsilon - \E[f(X)]}_{\psi_2}
\leq 2 \norm{\Sigma}^{1/2} + 1\,,
\end{align*}
which yields \eref{(b)}.
\end{proof}

\begin{lemma}\label{lem:gaussian}
Let $W$ have law $\cN(\zeros, \id)$ and $A$ be positive definite and $Z$ have law $\cN(\mu, \Sigma)$.
Then the following hold:
\begin{enumerate}
\item $\E[\norm{W}^2_A] = \tr(A)$ and $\E[\norm{W}^4_A] = \tr(A)^2 + 2 \tr(A^2)$.
\item $\E[\tr(A (WW^\top - \id) A (WW^\top - \id))] = \tr(A)^2 + \tr(A^2)$.
\item $\E[\tr(A (WW^\top - \id)(WW^\top - \id)A)] = (d^2+2d -1)\tr(A^2)$.
\item $\E[\norm{Z}^2] = \tr(\Sigma) + \norm{\mu}^2$.
\item $\E[\norm{Z}^4] = \tr(\Sigma)^2 + 2 \tr(\Sigma^2) + \norm{\mu}^2_{\Sigma} + \norm{\mu}^4 + 2\tr(\Sigma) \norm{\mu}^2 \le 3 (\E[\norm{Z}^2])^2$.
\end{enumerate}
\end{lemma}

\begin{proof}
By rotational invariance of the Gaussian and because $A$ is diagonalised by a rotation matrix it suffices to consider the case
that $A$ is diagonal with eigenvalues $(\lambda_m)_{m=1}^d$.
Then
\begin{align*}
\E[\norm{W}^2_A] = \E\left[\sum_{m=1}^d \lambda_m W_m^2\right] = \sum_{m=1}^d \lambda_m = \tr(A) \,.
\end{align*}
Furthermore, denoting the eigenvalues of $A$ by $\lambda_1,\ldots,\lambda_d$, we have
\begin{align*}
\E[\norm{W}_A^4]
= \E\left[\left(\sum_{m=1}^d \lambda_m W_m^2\right)^2\right] 
= \tr(A)^2 + \sum_{m=1}^d \lambda_m^2 (\E[W_m^4] - 1) 
= \tr(A)^2 + 2 \tr(A^2)\,,
\end{align*}
where in the final equality we used the fact that the fourth moment of a standard Gaussian is $\E[W_m^4] = 3$. 
For \eref{(b)},
\begin{align*}
\E[\tr(A (WW^\top - \id) A (WW^\top - \id))]
&= \E[\norm{W}_A^4 - 2\tr(A WW^\top A) + \tr(A^2)] \\
&= \E[\norm{W}_A^4] - \tr(A^2) \\
&= \tr(A)^2 + \tr(A^2)\,,
\end{align*}
where in the final equality we used part \eref{(a)}.
Part \eref{(c)} follows similarly:
\begin{align*}
\E[\tr(A (WW^\top - \id)(WW^\top - \id)A)]
&= \E[\norm{W}^4 \tr(A^2) - 2\tr(A WW^\top A) + \tr(A^2)] \\
&= (d^2+2d -1) \tr(A^2)\,,
\end{align*}
where we used the second statement of part \eref{(a)} in the last step.
Parts \eref{(d)} and \eref{(e)} follow by noting that $Z$ and $\Sigma^{1/2} W + \mu$ have the same law and using parts \eref{(a)} and \eref{(b)}.
In particular,
\begin{align*}
\E[\norm{Z}^2] &= \E[\norm{\Sigma^{1/2} W + \mu}^2] = \E[\|W\|_\Sigma^2] + \norm{\mu}^2 = \tr(\Sigma) + \norm{\mu}^2~.
\end{align*}
Furthermore,
\begin{align*}
\E[\norm{Z}^4] & = \E\left[\left((W^\top \Sigma^{1/2} + \mu^\top)(\Sigma^{1/2} W +\mu)\right)^2\right] \\
& = \E\left[ \left(\norm{W}^2_\Sigma + 2 W^\top \Sigma^{1/2} \mu + \norm{\mu}^2\right)^2\right] \\
& = \E\left[\norm{W}^4_\Sigma + \mu^\top \Sigma^{1/2} W W^\top \Sigma^{1/2} \mu + \norm{\mu}^4 + 2 \norm{W}^2_\Sigma \norm{\mu}^2 \right] \\
& = \tr(\Sigma)^2 + 2 \tr(\Sigma^2) + \norm{\mu}^2_{\Sigma} + \norm{\mu}^4 + 2\tr(\Sigma) \norm{\mu}^2 \\
& \le \tr(\Sigma)^2 + 2 \tr(\Sigma^2) + \norm{\mu}^4 + 3\tr(\Sigma) \norm{\mu}^2 \\
& \le 1.5 (\tr(\Sigma)+\norm{\mu}^2)^2 +  1.5 \tr(\Sigma)^2 \\
& \le 3 (\E[\norm{Z}^2])^2,
\end{align*}
where in the last step we used part \eref{(d)}.
\end{proof}

\end{document}